\documentclass[11pt]{article} % For LaTeX2e
\usepackage{url}
\usepackage{smile}
\usepackage{graphicx} % more modern
\usepackage{algorithm}
\usepackage{algorithmic}
\usepackage{todonotes}
\usepackage{epstopdf}
\usepackage{wrapfig}
\usepackage[colorlinks, linkcolor=blue, anchorcolor=blue, citecolor=blue]{hyperref}
\usepackage[margin=1in]{geometry}
\usepackage[normalem]{ulem}
\usepackage[export]{adjustbox}% http://ctan.org/pkg/adjustbox
\usepackage{mathtools, cuted}
\usepackage{natbib}
\usepackage{enumerate}
\usepackage{enumitem}
\usepackage{kpfonts}
\DeclareMathAlphabet{\mathsf}{OT1}{cmss}{m}{n}

\SetMathAlphabet{\mathsf}{bold}{OT1}{cmss}{bx}{n}

\providecommand{\norm}[1]{\|#1\|}

\newtheorem*{theorem*}{Theorem}

\title{\huge \bf Effective Minkowski Dimension of Deep Nonparametric Regression:\\ Function Approximation and Statistical Theories}

\author{Zixuan Zhang ~ Minshuo Chen ~ Mengdi Wang ~ Wenjing Liao ~ Tuo Zhao\thanks{Zixuan Zhang and Tuo Zhao are affiliated with School of Industrial and Systems Engineering at Georgia Tech; Minshuo Chen and Mengdi Wang are affiliated with Electrical and Computer Engineering at Princeton University; Wenjing Liao is affiliated with School of Mathematics at Georgia Tech; Email: $\{$zzhang3105, wliao60, tourzhao$\}$@gatech.edu, $\{$mc0750, mengdiw$\}$@princeton.edu.}}

\newcommand{\commentout}[1]{}

\begin{document}

\maketitle

\begin{abstract}

Existing theories on deep nonparametric regression have shown that when the input data lie on a low-dimensional manifold, deep neural networks can adapt to the intrinsic data structures. In real world applications, such an assumption of data lying exactly on a low dimensional manifold is stringent. This paper introduces a relaxed assumption that the input data are concentrated around a subset of $\RR^d$ denoted by $\cS$, and the intrinsic dimension of $\cS$ can be characterized by a new complexity notation -- effective Minkowski dimension. We prove that, the sample complexity of deep nonparametric regression only depends on the effective Minkowski dimension of $\cS$ denoted by $p$. We further illustrate our theoretical findings by considering nonparametric regression with an anisotropic Gaussian random design $N(0,\Sigma)$, where $\Sigma$ is full rank. When the eigenvalues of $\Sigma$ have an exponential or polynomial decay, the effective Minkowski dimension of such an Gaussian random design is $p=\cO(\sqrt{\log n})$ or $p=\cO(n^\gamma)$, respectively, where $n$ is the sample size and $\gamma\in(0,1)$ is a small constant depending on the polynomial decay rate. Our theory shows that, when the manifold assumption does not hold, deep neural networks can still adapt to the effective Minkowski dimension of the data, and circumvent the curse of the ambient dimensionality for moderate sample sizes.
\end{abstract}

\section{Introduction}

Deep learning has achieved impressive successes in various real-world applications, such as computer vision \citep{krizhevsky2012imagenet, goodfellow2014generative, Long_2015_CVPR}, natural language processing \citep{graves2013speech, bahdanau2014neural, young2018recent}, and robotics \citep{gu2017deep}. One notable example of this is in the field of image classification, where the winner of the 2017 ImageNet challenge achieved a top-5 error rate of just 2.25\% \citep{hu2018squeeze} using a training dataset of 1 million labeled high resolution images in 1000 categories. Deep neural networks have been shown to outperform humans in speech recognition, with a 5.15\% word error rate using the LibriSpeech training corpus \citep{panayotov2015librispeech}, which consists of approximately 1000 hours of 16kHz read English speech from 8000 audio books.

The remarkable successes of deep learning have challenged conventional machine learning theory, particularly when it comes to high-dimensional data. Existing literature has established a minimax lower bound of sample complexity $n\gtrsim\epsilon^{-(2s+d)/s}$ for learning $s$-H\"{o}lder functions in $\RR^d$ with accuracy $\epsilon$ \citep{gyorfi2006distribution}. 
This minimax lower bound, however, is far beyond the practical limits. For instance, the images in the ImageNet challenge are of the resolution $224\times224=50176$, while the sample size of $1.2$ million is significantly smaller than the theoretical bound.

%Fortunately, real-world data sets often exhibit low-dimensional structures. Many images consist of projections of a three-dimensional object followed by some transformations, such as rotation, translation, and skeleton. This generating mechanism induces a small number of intrinsic parameters\cite{hinton2006reducing, osher2017low}. More broadly, visual, acoustic, textual, and many other types of data often have low-dimensional geometric structures due to rich local regularities, global symmetries, repetitive patterns, or redundant sampling \cite{tenenbaum2000global, roweis2000nonlinear,coifman2005geometric}. 

Several recent results have attempted to explain the successes of deep neural networks
%such observations by considering nonparametric regression 
by taking the low-dimensional structures of data into consideration\citep{chen2019efficient,chen2022nonparametric,nakada2020adaptive,liu2021besov,schmidt2019deep}. Specifically, \citet{chen2022nonparametric} shows that when the input data are supported on a $p$-dimensional Riemannian manifold embedded in $\RR^d$, deep neural networks can capture the  low-dimensional  intrinsic structures of the manifold. The sample complexity in \citet{chen2022nonparametric} depends on the intrinsic dimension $p$, which circumvents the curse of ambient dimension $d$; \citet{nakada2020adaptive} assumes that the input data are supported on a subset of $\RR^d$ with Minkowski dimension $p$, and establishes a sample complexity similar to \citet{chen2022nonparametric}. \citet{liu2021besov} considers a  classification problem, and show that convolutional residual networks enjoy similar theoretical properties  to \citet{chen2022nonparametric}.

Considering the complexity of real world applications, however, the assumptions of data lying exactly on a low-dimensional manifold or a set with low Minkowski dimension are stringent. To bridge such a gap between theory and practice, we consider a relaxed assumption that the input data $X$ are approximately supported on a subset of $\RR^d$ with certain low-dimensional structures denoted by $\cS$. Roughly speaking, there exists a sufficiently small $\tau$ such that we have $\PP(X\notin\cS)=\tau$, where $\cS$ can be characterized by a new complexity notation -- effective Minkowski dimension. We then prove that under proper conditions, the sample complexity of nonparametric regression using deep neural networks only depends on the effective Minkowski dimension of $\cS$ denoted by $p$. Our assumption arises from practical motivations: The distributions of real-world data sets often exhibit a varying density. In practice, the low-density region can be neglected, if our goal is to minimize the $L_2$ prediction error in expectation. 

Furthermore, we illustrate our theoretical findings by considering nonparametric regression with an anisotropic multivariate Gaussian randomly sampled from $N(0,\Sigma)$ design in $\RR^d$. Specifically, we prove that when the eigenvalues of $\Sigma$ have an exponential decay, we can properly construct $\cS$ with the effective Minkowski dimension $p=\min(\cO(\sqrt{\log n}),d)$. Moreover, when the eigenvalues of $\Sigma$ have a polynomial decay, we can properly construct $\cS$ with the effective Minkowski dimension $p=\min(\cO(n^\gamma,d))$, where $\gamma\in(0,1)$ is a small constant. Our proposed effective Minkovski dimension is a non-trivial generalization of the manifold intrinsic dimension \citep{chen2022nonparametric} or the Minkowski dimension  \citep{nakada2020adaptive}, as both the intrinsic dimension or Minkowski dimension of the aforementioned $\cS$'s are $d$, which can be significantly larger than $p$ for moderate sample size $n$.

An ingredient in our analysis is an approximation theory of deep ReLU networks for $\beta$-Hölder functions \citep{YAROTSKY2017103,nakada2020adaptive,chen2019efficient}. Specifically, we show that, in order to uniformly approximate $\beta$-Hölder functions on a properly selected $\cS$ up to an $\epsilon$ error, %which is comparable to the statistical error of nonparametric regression, 
the network consists of at most $O(\epsilon^{-p/\beta})$ neurons and weight parameters, where $p$ is the effective Minkowski dimension of the input data distribution. The network size in our theory only weakly depends on the ambient dimension $d$, which circumvents the curse of dimensionality for function approximation using deep ReLU networks. Our approximation theory is established for the $L^2$ norm instead of the  $L^\infty$ norm in \citet{nakada2020adaptive,chen2019efficient}. The benefit is that we only need to approximate the function accurately on the high-density region, and allow for rough approximations on the low-density region. Such flexibility is characterized by our effective Minkowski dimension.

%which also involve function approximation theory of deep ReLU networks: Chen et al. 2021 and Nakada et al. 2021 requires deep ReLU networks to achieve an accurate point-wise approximation to the target function on the entire support, while our theory only requires deep ReLU networks to approximate the function within a properly selected subset of the support.

The rest of this paper is organized as follows: Section 2 reviews the background; Section 3 presents our functional approximation and statistical theories; Section 4 provides an application to Gaussian random design; Section 5 presents the proof sketch of our main results; Section 6 discusses related works and draws a brief conclusion.

\paragraph{Notations} Given a vector $v=(v_1,...,v_d)^\top\in\RR^d$, we define $\norm{v}_p^p=\sum_{j}|v_j|^p$ for $p\in[1,\infty)$ and $\norm{v}_{\infty}=\max_j|v_j|$. Given a matrix $W=[W_{ij}]\in\RR^{n\times m}$, we define $\norm{W}_{\infty}=\max_{i,j}|W_{ij}|$. We define the number of nonzero entries of $v$ and $W$ as $\norm{v}_0$ and $\norm{W}_0$, respectively. For a function $f(x)$, where $x\in\cX\subseteq \RR^d$, we define $\norm{f}_{\infty} = \max_{x\in\cX}|f(x)|$. We define $\norm{f}^2_{L^2(P)}=\int_{\cX}f^2(x)p(x)dx$, where $P$ is a continuous distribution defined on $\cX$ with the pdf $p(x)$.

%\fbox{Put a main theorem here?}

\section{Background}

In nonparametric regression, the aim is to estimate a ground-truth regression function $f^*$ from i.i.d. noisy observations $\left\{(x_i, y_i)\right\}_{i=1}^n$. The data are generated via
\begin{align*}
y_i = f^*(x_i) + \xi_i,
\end{align*}
where the noise $\xi_i$'s are i.i.d. sub-Gaussian noises with $\EE[\xi_i]=0$ and variance proxy $\sigma^2$, which are independent of the $x_i$'s. To estimate $f^*$, we minimize the empirical quadratic loss over a concept class $\cF$, i.e.,
\begin{align}\label{eq: obj}
\hat{f} \in \argmin_{f \in \cF} \frac{1}{2n} \sum_{i=1}^n \left(f(x_i) - y_i \right)^2.
\end{align}
We assess the quality of estimator $\hat{f}$ through bounding $L^2$ distance between $\hat{f}$ and $f^*$, that is,
\begin{align*}
\norm{\hat{f} - f^*}_{L^2(P_{\rm data})}^2 \leq \gamma(n).
\end{align*}
Here $\gamma(n)$ is a function of $n$ describing the convergence speed and $P_{\rm data}$ is an unknown sampling distribution ofthe  $x_i$'s supported on $\cD_{\rm data}$.

Existing literature on nonparametric statistics has established an optimal rate of $\gamma(n) \lesssim n^{-\frac{2\alpha}{2\alpha + d}}$, when $f^*$ is $\alpha$-smooth with bounded functional norm, and $\cF$ is properly chosen \citep{wahba1990spline, altman1992introduction, fan1996local, tsybakov2008introduction, gyorfi2006distribution}.

The aforementioned rate of convergence holds for any data distribution $P_{\rm data}$. For high-dimensional data, the convergence rate suffers from the curse of dimensionality. However, in many practical applications, $P_{\rm data}$ exhibits important patterns. For example, data are highly clustered in certain regions, while scarce in the rest of the domain. In literature, a line of work studies when $P_{\rm data}$ is supported on a low-dimensional manifold \citep{BickelLi, cheng2012local, liao2021multiscale, NIPS2011_4455, NIPS2013_5103, yang2015minimax}. The statistical rate of convergence $\gamma(n)$ in these works depends on the intrinsic dimension of the manifold, instead of the ambient dimension. Recently, neural networks are also shown to be able to capture the low-dimensional structures of data \citep{schmidt2019deep, nakada2020adaptive, chen2022nonparametric}.

As mentioned, aforementioned works assume that data exactly lie on a low-dimensional set, which is stringent. Recently, \citet{cloninger2020relu} relaxes the assumption such that data are concentrated on a tube of the manifold, but the radius of this tube is limited to the reach \citep{federer1959curvature} of the manifold. 
In this paper, we establish a fine-grained data dependent nonparametric regression theory, where data are approximately concentrated on a low-dimensional subset of the support.

%We impose the following assumptions.
%\begin{assumption}\label{assumption:domain}
%Without loss of generality, we assume $\cD_{\rm data} = [0, 1]^d$.
%\end{assumption}
% \begin{assumption}\label{assumption:pdata}
% Data sampling distribution has a continuous density $p_{\rm data}$ and is concentrated around clusters, i.e., for any $\eta > 0$, it holds
% \begin{align*}
% \int_{\cD_{\rm data}} \mathds{1}\{p_{\rm data}(x) > \eta\} dx \leq 
% \begin{cases}
% C (1+\eta)^{-k} & \textit{(Polynomial~decay)} \\
% C_1 \exp(-C_2\eta) & \textit{(Exponential~decay)}
% \end{cases},
% \end{align*}
% where $k$ is an integer and $C, C_1, C_2$ are constants.
% \end{assumption}
% Assumption \ref{assumption:pdata} says the Lebesgue measure of $\{p_{\rm data} > \eta\}$ decays relatively fast, indicating the well concentration of probability mass in $P_{\rm data}$.

% We assume that the data sampling distribution $P_{\rm data}$ is concentrated in certain regions in $\cD_{\rm data}$.
To facilitate a formal description, we denote $\cD_{\rm data}$ as the data support. Given $r, \tau > 0$, we define
\begin{align*}
N(r; \tau)\! :=\! \inf_{S} \{N_r(S)\!:\! S\! \subset \! \cD_{\rm data}\! ~\text{with}~\! P_{\rm data}(S)\! \geq 1 - \tau\},
\end{align*}
where $N_r(S)$ is the $r$-covering number of $S$ with respect to $L^\infty$ distance.
% the monotonic function $t(\cdot)$ quantifies the total probability mass of $P_{\rm data}$ outside $S$.
\begin{assumption}\label{assumption:pdata_cover}
For any sufficiently small $r, \tau > 0$, there exists a positive constant $p=p(r,\tau)$ such that  
\begin{align*}
\frac{\log N(r; \tau)}{-\log r} \leq p(r,\tau).
\end{align*}
Furthermore,  there exists $S \subset \cD_{\rm data}$ such that $$N_r(S) \leq c_0 N(r;\tau) \leq c_0 r^{-p}$$ for some constant $c_0>1$, $P_{\rm data}(S^c) \leq \tau$ and $| x_i| \leq R_S$ for any $x=(x_1,\ldots,x_d)\in S$ and some constant $R_S>0.$
\end{assumption}

%\textcolor{blue}{Here $\tau$ characterizes the probability outside some chosen region $\cS$ and $r$ represents the covering accuracy. Every pair of $\tau$ and $r$ determines a $p$. As $\tau$ decreases, the volume of $\cS$ tends to be larger and the covering number becomes larger. Moreover, as $r$ gets smaller, the covering number also increases, which could result in a greater effective Minkowski dimension.}

%\begin{example}[Manifold data]\label{ex:manifold}
%\end{example}

%\begin{example}[Mixture model]\label{ex:mixture}
%\end{example}

%{\color{blue} Wenjing notes:} 1. To make this assumption adapt the the function regularity, we can assume that $f$ is $s$-Holder in $S_\delta$, where $S_\delta$ is a slightly larger set than $S$. For a fixed $\delta>0$, we define $S_\delta = \{x:\ {\rm dist}(x,S)<\delta\}.$ This $\delta$ can be fixed or dependent on $\tau$. Now the smoothness $s$ can be dependent on $\tau$ such that $s=s(\tau)$. This will relax the smoothness parameter. We should be able to find examples that when $s(\tau)$ increases as the region $S$ shrinks, the rate will be improved.

We next introduce Hölder functions and the Hölder space. 

%Then we present the network architecture used to adaptively approximate the function in Hölder Space.

\begin{definition}[Hölder Space]\label{defi: holder}
Let $\beta >0$ be a degree of smoothness. For $f:\cX \to \RR$, the \emph{Hölder norm} is defined as
\begin{align*}
    \norm{ f}_{\cH(\beta, \cX)} 
    := \max_{\alpha:\norm{\alpha}_1 < \lfloor \beta \rfloor} \sup_{x \in \cX} |\partial^\alpha f(x) | + \max_{\alpha:\norm{\alpha}_1 = \lfloor \beta \rfloor} \sup_{x,x' \in \cX, x \neq x'} \frac{ |\partial^\alpha f(x) -\partial^\alpha f(x')| }{\norm{x-x'}_\infty^{\beta - \lfloor \beta \rfloor}}.
\end{align*}
Then the \emph{Hölder space} on $\cX$ is defined as
\begin{equation*}
    \cH(\beta,\cX) = \big\{ f \in C^{\lfloor \beta \rfloor }(\cX) \big| \norm{f}_{\cH(\beta,\cX)} \leq 1  \big\}.
\end{equation*}
%Also, we denote $\mathcal{H}(\beta,\cX,M)= \{ f \in \mathcal{H}(\beta,\cX) \big|  \norm{f}_{\mathcal{H}(\beta, \cX)} \leq M  \}$ as the $M$-radius closed ball in $\mathcal{H}(\beta,\cX)$ .
\end{definition}
Without loss of generality, we impose the following assumption on the target function $f^*$:
\begin{assumption}\label{assump:function}
    The ground truth function $f^*:\cD_{\rm data} \to \RR$ belongs to the Hölder space $\cH(\beta,\cD_{\rm data})$ with $\beta \in (0,d)$.  
\end{assumption}
Although the Hölder norm of $f^*$ is assumed to be bounded by $1$, our results can be easily extended to the case when $\norm{f^*}_{\cH(\beta, \cD_{\rm data})}$ is upper bounded by any positive constant. In addition, $\beta < d$ is a natural assumption. Given that ambient dimension $d$ is always large, it is unusual for regression functions to possess a degree of smoothness larger than $d$.

Our goal is to use multi-layer ReLU neural networks to estimate the function $f^*$. Given an input $x$, an $L$-layer ReLU neural network computes the output as
\begin{equation}\label{eq: nn-form}
    f(x) =  W_L \cdot \relu (W_{L-1}\cdots \relu(W_1 x +b_1) \cdots +b_{L-1}) +b_L,
\end{equation}
where $W_1,\ldots,W_L$ and $b_1,\ldots,b_L$ are weight matrices and intercepts respectively. The $\relu(\cdot)$ activation function denotes the entrywise rectified linear unit, i.e. $\relu(a)=\max \{a,0\}$. The empirical risk minimization in \eqref{eq: obj} is taken over the function class $\mathcal{F}$ given by a network architecture. 

\begin{definition}[Function Class Given by a Network Architecture]\label{defi: nn-class}
    Given a tuple $(L,B,K)$, a functional class of ReLU neural networks is defined as follows:
    \begin{align*}
        \mathcal{F}(L,B,K)\! :=\!\big\{ f |& f(x) \text{ in the form of } \eqref{eq: nn-form} \text{ with } L \text{ layers}, \|f\|_\infty \leq 1, 
        \|W_i\|_\infty \leq B, ~~\|b_i\|_\infty \leq B\\
        &\text{ for } i=1,\ldots,L, \sum_{i=1}^L \|W_i\|_0 + \|b_i\|_0 \leq K \big\}.
    \end{align*}
%{\color{red} Define norms in the notation part.}
\end{definition}
\section{Approximation and Generalization Theory}\label{sec:approx}

In this section, we present generic approximation and generalization theory and defer detailed proofs to Section \ref{sec:proof-sketch-approx} and \ref{proof:sketch-generalization} respectively. Firstly, we introduce the approximation theory of utilizing deep neural networks to approximate Hölder functions. The approximation error is determined by effective Minkowski dimension of data distribution and probability of low-density area.  Furthermore, we present the generalization error when  approximating regression function $f^*$. The convergence rate also depends on effective Minkowski dimension. 
%A key property of the network architecture in Definition \ref{defi: nn-class} is its universal approximation ability. \citet{nakada2020adaptive} provided that ReLU neural networks can universally approximate a function in Hölder Space whose support has low intrinsic dimensionality within a proper size. We extend this result to show the approximation ability of these networks are adaptive to data distribution.

%The result is summarized in the following theorem:
%To establish the statistical guarantee of estimating $f^*$, we develop the following neural network approximation theory, which is used for bounding the bias of estimation.
\begin{theorem}[Approximation of deep neural networks]\label{lemma:approx}
    Suppose Assumption \ref{assumption:pdata_cover} hold. 
    For $\beta>0$ and any sufficiently small $\epsilon ,\tau >0$, consider a tuple $(L,B,K)$
    \begin{equation*}
        L=C_1,~ B=O(R_S^{\beta s} \epsilon^{-s}),~ \text{and} ~K=C_2 (R_S d)^p \epsilon^{-p/\beta},
    \end{equation*}
    where $R_S>0$ and $p=p(d^{-1}\epsilon^{1/\beta}/2,\tau)$ are given by Assumption \ref{assumption:pdata_cover},  and
    \begin{align*}
        C_1 = O(d),~
        C_2 = O\big( d^{2+\lfloor \beta \rfloor}\big),~\text{and} ~s=s(\beta).
    \end{align*}
    Then for any $f^* \in\cH(\beta, \cD_{\rm data})$, we have
    \begin{equation*}
        \inf_{f \in \mathcal{F}(L,B,K)} \norm{f-f^*}^2_{L^2(P_{\rm data})} \leq \epsilon^2 + 4\tau.
    \end{equation*}
\end{theorem}

The novelty of Theorem \ref{lemma:approx} is summarized below:

\textbf{Dependence on Effective Minkowski Dimension.} The approximation rate in Theorem \ref{lemma:approx} is $O(K^{-\beta/p})$, which only depends on effective Minkowski dimension $p < d$ and function smoothness $\beta$, but not on ambient dimension $d$. Compared to \citet{YAROTSKY2017103}, our results improves the exponential dependence of neural network size on $d$ to that on $p$. Moreover, unlike \citet{nakada2020adaptive} and \citet{chen2022nonparametric}, our results do not require that data distribution is exactly supported on a low-dimensional structure. Instead, our results can work for data distribution with high-dimensional support as long as its effective Minkowski dimension is relatively small.

\textbf{Relaxation to the $L^2$-error.} The approximation error in Theorem \ref{lemma:approx} is established with respect to the $L^2(P_{\rm data})$ norm, while most of existing works focus on the $L^\infty$ error \citep{YAROTSKY2017103,nakada2020adaptive,chen2019efficient}. Intuitively, it is not necessary for the network class to approximate the function value at each point in the domain $\cD_{\rm data}$ precisely when data distribution is highly concentrated at certain subset. Instead, it suffices to approximate $f^*$ where the probability density is significant, while the error for the low-density region can be easily controlled since the regression function $f^*$ and the neural network class $f \in \cF(L,B,K)$ are bounded.

The benefit of using the $L^2$ error is that, we only need to control the approximation error of $f^*$ within some chosen region $S\subseteq \cD_{\rm data}$. Here $S$ has an effective Minkowski dimension $p$, which ensures that it can be covered by $O(r^{-p})$ hypercubes with side length $r$. Then we design deep neural networks to approximate $f^*$ within each hypercube and thus the network size depends on the number of hypercubes used to cover $S$. This explains why network size in Theorem \ref{lemma:approx} depends on $p$. Meanwhile, the probability out of $S$ is negligible since the data density is low. We further demonstrate that this probability $\tau$ is far less than the approximation error in Section \ref{sec:example}. By this means, we succeed to reduce the network size and at the same time achieve a small $L^2$ approximation error. 

We next establish the generalization result for the estimation of $f^*$ using deep neural networks.
\begin{theorem}[Generalization error of deep neural networks]\label{thm:generalization}
Suppose Assumption \ref{assumption:pdata_cover} holds. Fix any sufficiently small $r,\tau >0$ satisfying $r<R_S$ and $\tau < r^{4\beta}/4$.  Set a tuple $(L,B,K)$ with $C_1,C_2$ and $s$ appearing in Theorem \ref{lemma:approx} as
\begin{equation*}
    L=C_1,~ B=O(R_S^{\beta s} r^{-\beta s}), ~\text{and}~K=C_2 R_S^p r^{-p}
\end{equation*}
with $p=p(r,\tau)$. Let $\hat{f}$ be the global minimizer of empirical loss given in \eqref{eq: obj} with the function class 
$\mathcal{F}=\mathcal{F}(L,B,K)$. Then we have 
\begin{align*}
    \EE \| \hat{f} - f^*\|^2_{L^2(P_{\rm data})} = O\biggl(   \tau +  \sigma r^{2\beta} + \frac{\sigma^2}{n}  \biggl(\frac{R_S}{r} \biggr)^p \log\biggl( \frac{(R_S/r)^p}{ r^{4\beta}-4\tau}\biggr)\biggr),
\end{align*}
where $O(\cdot)$ hides polynomial dependence on $d$.
%\begin{equation*}
%    \| \hat{f} - f^*\|^2_{L^2(P_{\rm data})} = \tilde{O} \bigg( \tau +  r^{2\beta} + \frac{1}{n} r^{-p} \log\bigg( \frac{1}{r}\bigg) + \frac{1}{n}\bigg),
%\end{equation*}
%holds with probability at least $1-2\exp(-n^{p/(2\beta+p)})$ for any $n \geq N$ with a sufficiently large $N$.
\end{theorem}

Theorem \ref{thm:generalization} is a statistical estimation result. It implies that the generalization error also depends on effective Minkowski dimension $p$. To establish this result, we decompose the squared error into a squared bias term and a variance term. The bias is tackled with the approximation error in Theorem \ref{lemma:approx} and the variance depends on the network size. With the network size growing, the variance term increases while the bias term decreases, since the approximation capability of neural networks is enhanced as the size of the network enlarges. Therefore, we need to trade off between the squared bias and the variance to minimize the squared generalization error.

Notably, our analysis in Section \ref{sec:approx} holds for any sufficiently small $\tau$ and $r$, and every pair of $\tau$ and $r$ determines a $p$. As shown in Assumption \ref{assumption:pdata_cover}, if $\tau$ and $r$ decreases, the covering number will become larger while the approximation can be more accurate. In order to establish an explicit bound, we need to trade off $\tau$ and $r$ for the given sample size $n$. Therefore the “optimal” $p$ eventually becomes functions of $n$. We call such an “optimal” $p$ effective Minkowski dimension.

In the next section, we give two specific classes of Gaussian random designs to illustrate how effective Minkowski dimension $p(r,\tau)$ scales with $r$ and $\tau$. We further show that, under a proper selection of the region $S$ and the covering accuracy $r$, the convergence rate for the estimation of $f^*$ using deep neural networks is $\tilde{O}(n^{-2\beta/(2\beta+p)})$, where the effective Minkowski dimension $p$ is properly chosen.

\section{Application to Gaussian Random  Design}\label{sec:example}

%In this section, we give two examples to illustrate that ReLU neural networks are adaptive to data distribution.
In literature, it is common to consider random Gaussian  design in nonparametric regression \citep{anderson1962introduction,muller2006linear,chatfield2018introduction}. In this section, we take anisotropic multivariate Gaussian design as example to justify Assumption \ref{assumption:pdata_cover} and demonstrate the effective Minkowski dimension. Here we only provide our main theorems and lemmas. The detailed proofs are given in Section \ref{sec:gaussian}. 

Consider a Gaussian distribution $P_{\rm data} \sim N(\bold{0},\Sigma)$ in $\RR^d$. The covariance matrix $\Sigma$ has the eigendecomposition form: $\Sigma =Q\Gamma Q^\top$, where $Q$ is an orthogonal matrix and $\Gamma = \diag(\gamma_1,\ldots,\gamma_d)$. For notational convenience in our analysis, we further denote eigenvalue $\gamma_i=\lambda_i^2$ for $i=1,\ldots,d$. Without loss of generality, assume that $\lambda_1^2\geq \lambda_2^2 \geq \ldots \geq \lambda_d^2$. Furthermore, we assume that $\{ \lambda_i^2 \}_{i=1}^d$ has an exponential or polynomial decay rate:
\begin{assumption}[Exponential decay rate]\label{assump:eigen-decay-exp}
    The eigenvalue series $\{\gamma_i \}_{i=1}^d=\{\lambda_i^2 \}_{i=1}^d$ satisfies $\lambda_i \leq \mu \exp\{-\theta i\}$ for some constants $\mu,\theta>0$.
\end{assumption}
\begin{assumption}[Polynomial decay rate]\label{assump:eigen-decay-poly}
    The eigenvalue series $\{\gamma_i \}_{i=1}^d=\{\lambda_i^2 \}_{i=1}^d$ satisfies $\lambda_i \leq \rho i^{-\omega}$ for some constants $\rho>0$ and $\omega > 1$.
\end{assumption}

When the eigenvalues decay fast, the support of the data distribution $P_{\rm data}$ has degeneracy in some directions. In this case, the majority of probability lies in some region $S \subset \RR^d$, which has an effective Minkowski dimension $p<d$. %it suffices to focus some $p$-dimensional subspace since the other $d-p$ dimensions degenerate in probability measure, where $p$ is the effective Minkowski dimension of data distribution. Therefore, we design $S$ in a way such that $P_{\rm data}(S)$ concentrates in  $p$ dimensions. %Intuitively, the number of hypercubes with side length $r$ used to cover $S$ is approximately $r^{-p}$.In literature, it is common to truncate the ambient space within a hyper-ellipsoid for Gaussian distribution. Similarly, we 
Specifically, consider a \lq\lq thick\rq\rq\ low-dimensional hyper-ellipsoid in $\RR^d$, 
\begin{align}
    S(R,r;p)\! :=\! \bigg\{ Qz \bigg| z\!=\!(z_1,\ldots,z_d)\in \RR^d, \sum_{i=1}^p \frac{z_i^2}{\lambda_i^2} \leq R^2, |z_j| \leq \frac{r}{2} \text{ for } j=p+1,\ldots,d  \bigg\}, \label{eq:set}
\end{align}
where $R,r>0$ and $p\in \NN_+$ are independent parameters. For the simplicity of notation, we first define a standard hyper-ellipsoid and then linearly transform it to align with the distribution $N(0,\Sigma)$.  The set $S(R,r;p)$ can be regarded as a hyper-ellipsoid scaled by $R>0$ in the first $p$ dimensions, and  with thickness $r>0$ in the rest $d-p$ dimensions. Then we construct a minimal cover as a union of nonoverlapping hypercubes with side length $r$ for $S(R,r;p)$.
The following lemma characterizes the relationship between the probability measure outside $S(R,r;p)$ and its covering number.

\begin{lemma}\label{lemma:gaussian}
    Given the eigenvalue series $\{\lambda_i^2\}_{i=1}^d$, for any $R,r>0$, choose $p>0$ such that $\lambda_p^{-1} = 2R/r$. If $p< R^2$, we will have
    \begin{equation*}
        \PP(X \notin S(R,r;p) ) =
        O \bigl(\exp( -R^2/3) \bigr),
    \end{equation*}
    \begin{equation*}
      N_r(S(R,r;p)) \leq \biggl(\frac{2R}{r} \biggr)^p \cdot\prod_{i=1}^p \lambda_i= \prod_{i=1}^p \bigg( \frac{\lambda_i}{\lambda_p}\bigg).
    \end{equation*}
    %where $N_r(S(R,r;p)$ denotes the covering number of $S(R,r;p)$ using hypercubes with side length $r$.
\end{lemma}
\begin{remark} 
Since data distribution $P_{\rm data}$ is supported on $\RR^d$, both the intrinsic dimension and the Minkowski dimension of $P_{\rm data}$ are $d$. However, Lemma \ref{lemma:gaussian} indicates that the effective Minkowski dimension of $P_{\rm data}$ is at most $p$. 
\end{remark}

According to Lemma \ref{lemma:gaussian}, if we choose scale $R>\sqrt{p}$ properly, the probability outside $S$ can be sufficiently small while the covering number of $S$ is dominated by $r^{-p}$, which gives that the effective Minkowski dimension of $P_{\rm data}$ is at most $p$. Moreover, under fast eigenvalue decays, the product of the first $p$ eigenvalues appearing in $N_r(S(R,r;p))$ is a small number dependent of $p$. In these cases, we specify the selection of $R$, $r$ and $p$ accordingly and show the effective Minkowski dimension is reduced to $p/2$ in Appendix \ref{sec:dim-gaussian}. 

Furthermore, we remark that the effective Minkowski dimension $p$  is not a fixed number given data distribution $P_{\rm data}$, but an increasing function of sample size $n$. As sample size $n$ increases, the estimation accuracy of $f^*$ is required to be higher, so that we are supposed to design more and smaller hypercubes to enable preciser estimation by neural networks. Besides, some of the $d-p$ dimensions are not negligible anymore and thereby become effective compared to the accuracy. Therefore, we need to incorporate more dimensions to be effective to achieve higher accuracy.

With this observation, we construct $S(R,r;p)$ such that its effective Minkowski dimension $p(n)$ increases while thickness $r(n)$ decreases as sample size $n$ grows to enable preciser estimation. Then we develop the following sample complexity:

\begin{theorem}[Generalization error under fast eigenvalue decay]\label{thm:generalization-exp}
Under Assumption \ref{assump:function}, let $\hat{f}$ be the global minimizer of empirical loss given in \eqref{eq: obj} with function class 
$\mathcal{F}=\mathcal{F}(L,B,K)$. Suppose Assumption \ref{assump:eigen-decay-exp} hold. Set a tuple $(L,B,K)$ with the constants $C_1,C_2$ and $s$ appearing in Theorem \ref{lemma:approx} as 
\begin{equation*}
\begin{aligned}
    L= C_1,~~B = O\biggl(n^{\frac{\beta s}{2\beta+\sqrt{\log n/\theta}}} (\log n)^{\beta s} \biggr),~
    \text{and}~K = C_2 n^{\frac{\sqrt{\log n/\theta}}{2\beta + \sqrt{\log n/\theta} } }.
\end{aligned}
\end{equation*}
Then  we have 
\begin{equation*}
    \EE \norm{ \hat{f} - f^*}^2_{L^2(P_{\rm data})} = O\Bigl( \sigma^2 n^{- \frac{2\beta(1-\eta)}{2\beta+\sqrt{\log n/\theta}}}(\log n)^{3/2} \Bigr)
\end{equation*}
for sufficiently large $n$ satisfying $\log(\log n) / \sqrt{\theta \log n} \leq \eta$, where $\eta >0$ is an arbitrarily small constant.
Moreover, suppose Assumption \ref{assump:eigen-decay-poly} hold instead.  Set a tuple $(L,B,K)$ as \begin{equation*}
\begin{aligned}
        L= C_1,~~B&=O\Bigl(n^{\frac{(1+1/\omega)\beta s}{2\beta+ n^{\kappa}} } \Bigr),~\text{and}~K=C_2 \Bigl(n^{\frac{(1+1/\omega) n^{\kappa/(2\beta+n^\kappa)}}{4\beta + 2n^\kappa}} \Bigr),
\end{aligned}
\end{equation*}  
where $\kappa=(1+1/\omega)/\omega$. Then we have 
\begin{equation*}
     \EE \norm{ \hat{f} - f^*}^2_{L^2(P_{\rm data})} = O\Bigl( \sigma^2 n^{- \frac{2\beta}{2\beta+ n^{(1+1/\omega)/\omega}}} \log n\Bigr).
\end{equation*}
\end{theorem}

%\begin{theorem}[Generalization error under polynomial eigenvalue decay]\label{thm:generalization-poly}
%Fix any $f^*\in \mathcal{H}(\beta,[0,1]^d,M)$.  Set a tuple $(R,L,B,K)$ with the constants $C_1,C_2$ and $s$ appearing in Lemma \ref{lemma:approx} as $R=\|f^*\|_{\infty}$, $L=C_1$, $B=O(n^{2\beta s/(2\beta+p)})$ and $K=C_2n^{p/(2\beta+p)}$. Take the functional class 
%$\mathcal{F}=\mathcal{F}(R,L,B,K)$ in \eqref{eq: obj}. Then there exists a constant %$C=C(R,\beta,d,p,s,M,\sigma)$ such that 
%\begin{equation*}
%    \| \hat{f} - f^*\|^2_{L^2(P_{\rm data})} \leq C n^{-\frac{2\beta}{2\beta+n^{(1+\omega)/\omega^2}}} \log n.
%\end{equation*}
%holds with probability at least $1-2\exp(-n^{p/(2\beta+p)})$ for any $n \geq N$ with a sufficiently large $N$.
%\end{theorem}
Theorem \ref{thm:generalization-exp} suggests the effective Minkowski dimension of Gaussian distribution is $\sqrt{\log n/\theta}$ under exponential eigenvalue decay with speed $\theta$ and effective Minkowski dimension is $n^{(1+1/\omega)/\omega}$ under polynomial eigenvalue decay with speed $\omega$. For moderate sample size $n$, i.e. effective Minkowski dimension is less that ambient dimension $d$, Theorem \ref{thm:generalization-exp} achieves a faster convergence rate.  When we have a vast amount of data, the effective Minkowski dimension is the same as the ambient dimension $d$, and then we can apply standard analysis of deep neural networks for $d$-dimensional inputs to obtain the convergence rate  $\tilde{O}(n^{-2\beta/(2\beta+d)})$. To the best of our knowledge, Theorem \ref{thm:generalization-exp} appears to be the first result for nonparametric regression and deep learning theory, where the effective dimension varies with the sample size.

%\subsection{Application of Mixture Model}

\section{Proof Sketch}

This section contains proof sketches of Theorem \ref{lemma:approx}, \ref{thm:generalization} and Lemma \ref{lemma:gaussian}. 

\subsection{Proof Sketch of Theorem \ref{lemma:approx}}\label{sec:proof-sketch-approx}

We provide a proof sketch of Theorem \ref{lemma:approx} in this part and defer technical details of the proof to Appendix \ref{appendix:approx}.
The ReLU neural network in Theorem \ref{lemma:approx} is constructed in the following 5 steps:
\begin{enumerate}%[itemsep=1 pt,topsep=1 pt,parsep=1 pt]
    \item Choose region $S \subset \cD_{\rm data}$ only on which we use ReLU neural networks to approximate $f^*$.
    \item Construct a covering of $S$ with hypercubes and then divide these hypercubes into several groups, so that neural networks constructed with respect to each group have nonoverlapping supports.
    \item Implement ReLU neural networks to assign given input and estimated function value to corresponding hypercube.
    \item Approximate $f^*$ by a Taylor polynomial and then implement a ReLU neural network to approximate Taylor polynomial on each hypercube.
    \item Sum up all the sub-neural-networks and take maximum to approximate $f^*$.
\end{enumerate}

\textbf{Step 1. Space separation.} Firstly, we divide $\cD_{\rm data}$ into some region $S \subset \cD_{\rm data}$ with high probability measure and $S^c = \cD_{\rm data} $ with large volume. By Assumption \ref{assumption:pdata_cover}, for any sufficiently small $r,\tau>0$ and some constant $c_0>1$, there exists $S \subset \cD_{\rm data}$ such that $N_r(S) \leq c_0 N(r;\tau) \leq c_0 r^{-p}$ for some positive constant $p=p(r,\tau)$ and $P_{\rm data}(S^c) \leq \tau$. Intuitively, we only need to approximate $f^*$ on $S$ while $S^c$ is negligible due to its small probability measure.  Therefore, in the following steps, we only design a covering for $S$ and approximate $f^*$ in each hypercube of the covering. 

\textbf{Step 2. Grouping hypercubes.} Let $\cC$ be a minimum set of hypercubes with side length $r$ covering $S$. Then we partition $\cC$ into $\cC_1,\ldots,\cC_J$ such that each subset $\cC_j$ is composed of hypercubes separated by $r$ from each other. Lemma \ref{lem:covering_division} shows that the number of $\cC_j$'s is at most a constant dependent of $d$.

As a consequence, we group hypercubes into several subsets of $\cC$ so that  constructed neural networks with respect to each hypercube in $\cC_j$ have nonoverlapping support.

\textbf{Step 3. Hypercube Determination.} 
This step is to assign the given input $x$ and estimated function value $y$ to the hypercube where they belong. To do so, we design a neural network to approximate function $(x,y) \mapsto y\ind_I(x)$ where $I \in \cC$ is some hypercube. To make functions positive, we firstly consider approximating $f_0 = f^* + 2$. Notice that $f_0 \in \cH(\beta,\cD_{\rm data}, 3)$ and $1 \leq f_0(x) \leq 3 $ for any $x \in \cD_{\rm data}$. 

For any fixed $I \in \cC$, we define the center of $I$ as $(\iota_1, \dots, \iota_d)$.
Then we construct a neural network $g_I^{\text{ind},r}: \cD_{\rm data} \times \RR_\geq \to \RR_\geq$ with the form:
\begin{align}
    g_I^{\text{ind},r} (x,y) =  4 \relu \bigg(\sum_{i=1}^d \hat{\ind}_{I, i}^r(x_i) + \frac{y}4 - d \bigg),\label{eq:identical}
\end{align}
where $\hat{\ind}^r_{I, i}: \RR \to [0, 1]$ is the approximated indicator function given by
\begin{equation*}
  \hat{\ind}^r_{I, i}(z) = \begin{cases}
  %0 & \text{if } z \leq \iota_i - r,\\
  \frac{z - (\iota_i - r)}{r/2} & \text{if } \iota_i - r < z \leq \iota_i - \frac{r}{2},\\
  1 & \text{if } \iota_i - \frac{r}{2} < z \leq \iota_i + \frac{r}{2},\\
  \frac{(\iota_i + r) - z}{r/2} & \text{if } \iota_i + \frac{r}{2} < z \leq \iota_i + r,\\
  0 & \text{otherwise. } %\iota_i+r < z.
\end{cases}
\end{equation*}
We claim that neural network $g_I^{\text{ind},r}$ approximates function $(x,y) \mapsto y\ind_I(x)$. Moreover, Appendix \ref{sec:realize-ind} provides the explicit realization of $g_I^{\text{ind},r}$ by selecting specific weight matrices and intercepts.

\textbf{Step 4. Taylor Approximation.} 
In each cube $I \in \cC$, we locally approximate $f^*$ by a Taylor polynomial of degree $\lfloor\beta\rfloor$ and then we define a neural network to approximate this Taylor polynomial. Firstly, we cite the following lemma to evaluate the difference between any $\beta$-H\"{o}lder function and its Taylor polynomial:
\begin{lemma}[Lemma A.8 in \citet{PETERSEN2018296}]\label{lemma: taylor_approx}
  Fix any $f \in \cH(\beta, \cD_{\rm data})$ with $\norm{f}_{\cH(\beta, \cD_{\rm data})} \leq 1$ and $\bar{x} \in S$.
  Let $\bar{f}(x)$ be the Taylor polynomial of degree $\lfloor\beta\rfloor$ of $f$ around $\bar{x}$, namely,
  \begin{equation*}
    \bar{f}(x) = \sum_{ |\alpha| \leq \lfloor\beta\rfloor} \frac{\partial^\alpha f(\bar{x})}{\alpha!} (x - \bar{x})^\alpha.
  \end{equation*}
  Then, $|f(x) - \bar{f}(x)| \leq d^\beta \norm{x - \bar{x}}^\beta$ holds for any $x \in \cD_{\rm data} $.%, where $\underline{C}$ is a constant depending on $\beta, D$.
\end{lemma}

Next, we design an $m$-dimensional multiple output neural network $g^{\poly}_\epsilon=(g^{\poly}_{\epsilon,1},\dots,g^{\poly}_{\epsilon,m})$ to estimate multiple Taylor polynomials in each output. The existence of such neural network is ensured in the following lemma, which is a straightforward extension of Lemma 18 in \citet{nakada2020adaptive}.

\begin{lemma}[Taylor approximation on S]\label{lemma:poly}
  Fix any $m \in \NN_+$. Let $\{c_{k, \alpha}\} \subset [-1, 1]$ for $1 \leq k \leq m$. Let $\{x_k\}_{k=1}^m \subset S$.
  Then there exist $c_1^{\poly} = c_1^{\poly}(\beta, d, p)$, $c_2^{\poly} = c_2^{\poly}(\beta, d, p)$ and $s^{\poly}_1 = s^{\poly}_1(\beta, d, p)$  such that for any sufficiently small $\epsilon>0$, there is a neural network $g^{\poly}_\epsilon$ which satisfies the followings:
  \begin{enumerate}%[itemsep=1.3pt,topsep=1.3pt,parsep=1.2pt]
    \item $\sup_{x \in S} \bigl| g^{\poly}_{\epsilon,k}(x) - \sum_{|\alpha| < \beta} c_{k, \alpha} (x - x_k)^\alpha \bigr| \leq \epsilon$ for any $k=1,\ldots,m$,
    \item $L(g^{\poly}_\epsilon) \leq 1 + (2 + {\log_2\beta})(11 + (1+\beta)/p)$,
    \item $B(g^{\poly}_\epsilon) \leq c_1^{\poly} R_S^{\beta s_1^{\poly}}\epsilon^{-s_1^{\poly}}$,
    \item $K(g^{\poly}_\epsilon) \leq c_2^{\poly} (R_S^p \epsilon^{-p/\beta} + m)$.
  \end{enumerate}
\end{lemma}

For any cube $I_k \in \cC$, we take $f_{I_k}(x)$ as a Taylor polynomial function as with setting $\bar{x} \leftarrow x_{I_k}$ and $f \leftarrow f_0$ in Lemma \ref{lemma: taylor_approx}. Then we define a neural network to approximate $f_{I_k}$, which is an $\epsilon/2$-accuracy Taylor polynomial of $f_0$. Let $g^{\poly}_{\epsilon/2}$ be a neural network constructed in Lemma \ref{lemma:poly} with $\epsilon \leftarrow \epsilon/2$, $m \leftarrow N_r(S)$, $(x_k)_{k=1}^{m} \leftarrow (x_{I_k})_{k=1}^{N_r(S)}$, and
$(c_{k, \alpha})_{k=1}^{m} \leftarrow (\partial^\alpha f(x_{I_k})/\alpha!)_{k=1}^{N_r(S)}$ appearing in Lemma \ref{lemma:poly}.
Then, we obtain
\begin{align}
  \sup_{k=1,\dots,N_r(S)} \sup_{x \in S} \bigl| f_{I_k}(x) - g^{\poly}_{\epsilon/2,k}(x) \bigr| \leq \frac{\epsilon}{2}. \label{eq:poly}
\end{align}

In addition, we construct a neural network to aggregate the outputs of $g^{\poly}_{\epsilon/2}$. Define a neural network $g^{\text{filter}}_k : \RR^{d + N_r(S)} \to \RR^{d+1}$ which picks up the first $d$ inputs and $(d + k)$-th input as
\begin{equation*}
    g^{\text{filter}}_k (z)= \begin{pmatrix}\II_d & e_k^\top \\ \bf{0}_d & e_k^\top \end{pmatrix}z,~\text{ for } k=1,\dots,N_r(S).
\end{equation*}
Then we design a neural network $g^{\text{simul}}_{\epsilon/2}: \RR^d \to \RR^{N_r(S)}$ that simultaneously estimates Taylor polynomial at each cube. Specifically, $g^{\text{simul}}_{\epsilon/2}$ is formulated as below  
\begin{equation}\label{eq:simul}
\begin{split}
    g^{\text{simul}}_{\epsilon/2} = \bigl(g_{I_1}^{\text{ind},r}\circ g^{\text{filter}}_1 ,\ldots,g_{I_{ N_r(S)}}^{\text{ind},r}
    \circ g^{\text{filter}}_{ N_r(S)} \bigr) \circ (g^{Id}_{d,L},  g^{\poly}_{\epsilon/2}),
\end{split}
\end{equation}
where $g^{Id}_{d,L}: \RR^d \to \RR^d$ is the neural network version of the identity function whose number of layers is equal to $L(g^{\poly}_{\epsilon/2})$.

\textbf{Step 5. Construction of Neural Networks.} 
In this step, we construct a neural network $g^{f_0}_\epsilon$ to approximate $f_0 = f^* + 2 $. Let $g^{\max,5^d}$ be the neural network version of the maximize function over $5^d$ numbers. Besides, define
\begin{equation*}
    g^{\text{sum}}(z_1, \dots, z_{N_r(S)}) = \biggl( \sum_{I_k \in \cC_1} z_k, \dots,  \sum_{I_k \in \cC_{N_r(S)}} z_k \biggr),
\end{equation*}
which aims to sum up the output of $g^{\text{simul}}_{\epsilon/2}$ in each subset of covering $\cC_j$. 

Now we are ready to define $g^{f_0}_\epsilon$. Let $g^{f_0}_\epsilon:=g^{\max,5^d} \circ g^{\text{sum}} \circ g^{\text{simul}}_{\epsilon/2}$. Equivalently,  $g^{f_0}_\epsilon$ can be written as  $g^{f_0}_\epsilon=\max_{j \in [5^d]} \sum_{I_k \in \cC_j} g^{\text{simul}}_{\epsilon/2,k}$. Then we come to bound the approximation error of $g^{f_0}_\epsilon$. When $x \in S$, there exists some $I \in \cC$ such that $x \in I$.  Based on the method to construct neural networks, we have
\begin{align*}
    g^{f_0}_\epsilon(x) = \max_{I_k \in \neig(I)}  g^{\text{simul}}_{\epsilon/2,k}(x)  \leq\max_{I_k \in \neig(I)}  g^{\poly}_{\epsilon/2,k}(x),
\end{align*}
where $\neig(I)=\{ I' \in \cC | (I \oplus 3r/2)\cap I' \neq \emptyset \}$ denotes the $3r/2$-neighborhood of hypercube $I$.
In other words, when computing $g^{f_0}_\epsilon(x)$, we only need to take maximum over the estimated function value within hypercubes near $x$.
%For a further parameter tuning, we set $\gamma = D^{-1}\qty(3 M)^{-1/\beta} \varepsilon^{1/\beta}$.

Given sufficiently small $\epsilon > 0$, 
the error is bounded as
\begin{align*}
    | g^{f_0}_\epsilon (x) - f_0(x) | 
    \leq& \max_{I_k \in \neig(I)}  \bigl| g^{\poly}_{\epsilon/2,k}(x)-f_0(x) \bigr| \\
    \leq& \max_{I_k \in \neig(I)}  \bigl| g^{\poly}_{\epsilon/2,k}(x)-f_{I_k}(x) \bigr|+ \max_{I_k \in \neig(I)}  \bigl| f_{I_k}(x) -f_0(x) \bigr| \\
    \leq& \frac{\varepsilon}{2} + d^\beta \biggl(\frac{3r}{2} \biggr)^\beta \leq \epsilon,
\end{align*}
where the last inequality follows from \eqref{eq:poly} and Lemma \ref{lemma: taylor_approx}. Detailed derivation of approximation error is deferred to Appendix \ref{sec:approx-error-bound}. In terms of parameter tuning, we choose $r = d^{-1} \epsilon^{1/\beta}/2$. 

To extend results of $f_0$ to $f^*$, we implement a neural network $g^{\text{mod}}(z)=(-z+1)\circ\relu(-z+2)\circ \relu(z-1)$ and consider $g^{f^*}_\epsilon=g^{\text{mod}}\circ g^{f_0}_\epsilon$ to obtain the desired approximation error $\epsilon$ for any $x\in S$. Then we evaluate the approximation error with respect to $L^2$-norm:
\begin{align*}
    \Bigl\| g^{f^*}_\epsilon-f^* \Bigr\|_{L^2(P_{\rm data})} 
    =\biggl(\int_{S} +\int_{S^c} \biggr) \Bigl(g^{f^*}_\epsilon(x)-f^*(x) \Bigr)^2 \ud P_{\rm data}(x) \leq \epsilon^2 + 4\tau. %\norm{f^*}_\infty^2 .
\end{align*}
This follows from the aforementioned approximation error within $S$, boundedness of $f^*$ and neural networks, as well as the property that out-of-$S$ probability is upper bounded, i.e. $P_{\rm data}(S^c) \leq \tau$.

Finally, we sum up sizes of all the sub-neural-networks and thus obtain the network size of $g^{f^*}_\epsilon$. See Appendix \ref{sec:parameter} for detailed calculation.%Details of the proof for Theorem \ref{lemma:approx} are deferred to Appendix \ref{appendix:approx}.

\subsection{Proof Sketch of Theorem \ref{thm:generalization}}\label{proof:sketch-generalization}

Proof of Theorem \ref{thm:generalization} follows a standard statistical decomposition, i.e. decomposing the mean squared error of estimator $\hat{f}$ into a squared bias term and a variance term. We bound the bias and variance separately, where the bias is tackled by the approximation results given in Theorem \ref{lemma:approx} and the variance is bounded using the metric entropy arguments. 
%\citep{van1996weak, gyorfi2006distribution}. 
Details of the proof for Theorem\ref{thm:generalization} are provided in Appendix \ref{sec:proof-generalization}.
At first, we decompose the $L_2$ risk as follows:
\begin{align*}
\EE \norm{\hat{f}-f^*}^2_{L^2(P_{\rm data})} = \underbrace{2 \EE \biggl [\frac{1}{n}\sum_{i=1}^n (\hat{f}(x_i) - f^*(x_i))^2 \biggr]}_{T_1}  +\underbrace{\EE\norm{\hat{f}-f^*}^2_{L^2(P_{\rm data})} - 2 \EE\biggl [\frac{1}{n}\sum_{i=1}^n (\hat{f}(x_i) - f^*(x_i))^2 \biggr]}_{T_2},
\end{align*}
where $T_1$ reflects the squared bias of using neural networks to estimate $f^*$ and $T_2$ is the variance term.
% The proof of Theorem \ref{thm:stat} contains two parts for the estimates of $T_1$ and $T_2$, respectively. 

\textbf{Step 1. Bounding bias term $\bf{T_1}$.}
Since $T_1$ is the empirical $L_2$ risk of $\hat{f}$ evaluated on the samples $\{x_i\}_{i=1}^n$, we relate $T_1$ to the empirical risk by rewriting $f^*(x_i) = y_i - \xi_i$,  so that we can apply the approximation error to bound the minimal empirical risk achieved by $\hat{f}$. After some basic calculation, we have
\begin{align*}
T_1 %=& 2 \EE \left[\frac{1}{n} \sum_{i=1}^n (\hat{f}(x_i) - y_i + \xi_i)^2 \right] \nonumber \\
%\overset{(i)}{=}& 2 \EE \left[\frac{1}{n} \sum_{i=1}^n \left[(\hat{f}(x_i) - y_i)^2 + 2 \xi_i \hat{f}(x_i) - \xi_i^2 \right]\right] \nonumber \\
 %=& 2 \EE \biggl[\inf_{f \in \cF(R, L, B, K)} \frac{1}{n}\sum_{i=1}^n \Bigl[(f(x_i) - y_i)^2 \\
 %&\quad - \xi_i^2 + 2 \xi_i \hat{f}_n(\xb_i) \Bigr] \biggr] \nonumber \\
 %\overset{(ii)}
 \leq  2 \inf_{f \in \cF( L, B, K)} \norm{ f(x) - f^*(x)}^2_{L^2(P_{\rm data})}  + 4 \EE \left[\frac{1}{n} \sum_{i=1}^n \xi_i \hat{f}(x_i) \right].
\end{align*}
%Equality $(i)$ is obtained by expanding the square, where the cross term $\EE[\xi_i y_i] = \EE[\xi_i (f_0(\xb_i) + \xi_i)] = \EE[\xi_i^2]$ due to the independence between $\xb_i$ and $\xi_i$. Inequality $(ii)$ invokes the Jensen's inequalty to pass the expectation. To obtain term $(A)$, we expand $(f(\xb_i) - y_i)^2 = (f(\xb_i) - f_0(\xb_i) - \xi_i)^2$, and observe the cancellation of $-\xi_i^2$. 
Note that the first term is the squared approximation error of neural networks, which can be controlled by Theorem \ref{lemma:approx}. We bound the second term by quantifying the complexity of the network class $\cF( L, B, K)$. A precise upper bound of $T_1$ is given in the following lemma.
%whose proof follows a similar argument in \citet[Lemma 4]{schmidt2017nonparametric}.
\begin{lemma}\label{lemma:t1}
Fix the neural network class $\cF( L, B, K)$. For any $\delta \in (0, 1)$, there exists some constant $c>0$, such that
\begin{align*}
T_1 \leq &~ c \inf_{f \in \cF( L, B, K)} \norm{ f(x) - f^*(x)}^2_{L^2(P_{\rm data})} + c \sigma^2 \frac{\log \cN_2(\delta, \cF( L, B, K)) + 2}{n} \\
&\quad+  c\bigg(   \sqrt{\frac{\log \cN_2(\delta, \cF(L, B, K)) + 2}{n}} + 1 \bigg)\sigma \delta,
\end{align*}
where $\cN_2(\delta, \cF(L, B, K))$ denotes the $\delta$-covering number of $\cF(L, B, K)$ with respect to the $L^2$ norm, i.e., there exists a discretization of $\cF(L, B, K)$ into $\cN_2(\delta, \cF(L, B, K))$ distinct elements, such that for any $f \in \cF$, there is $\bar{f}$ in the discretization satisfying $\norm{\bar{f} - f}_2 \leq \epsilon$.
\end{lemma}

\textbf{Step 2. Bounding variance term $\bf{T_2}$.}
We observe that $T_2$ is the difference between the population risk of $\hat{f}$ and its empirical risk. However, bounding this difference is distinct from traditional concentration results because of the scaling factor $2$ before the empirical risk. To do this, we divide the empirical risk into two parts and use a higher-order moment (fourth moment) to bound one part. Using a Bernstein-type inequality, we are able to show that $T_2$ converges at a rate of $1/n$, and the upper bound for this is shown in the following lemma.
\begin{lemma}\label{lemma:t2}
For any $\delta \in (0, 1)$, there exists some constant $c'>0$, such that
\begin{align*}
T_2 \leq& \frac{c'}{3n} \log \cN_2(\delta/4H, \cF(L, B, K))+ c'\delta.
\end{align*}
\end{lemma}

\textbf{Step 3. Covering number of neural networks.}
The upper bounds of $T_1$ and $T_2$ in Lemma \ref{lemma:t1} and \ref{lemma:t2} both rely on the covering number of the network class $\cF(R, \kappa, L, p, K)$. In this step, we present an upper bound for the covering number $\cN_2(\delta, \cF(L, B, K) )$ for a given a resolution $\delta > 0$. %Since each weight parameter in the network is bounded by a constant $\kappa$, we construct a covering by partitioning the range of each weight parameter into a uniform grid. By choosing a proper grid size, we show the following lemma.
\begin{lemma}[Covering number bound for $\cF$ (Lemma 21 in \citet{nakada2020adaptive})]\label{lemma:covering-number}
  Given $\delta>0$, the $\delta$-covering number of the neural network class $\cF(L,B,K)$ satisfies
  \begin{align*}
    \log \cN_2(\delta, \cF( L, B,K)) \leq K \log\biggl(\! \frac{2^L \sqrt{dL} K^{L/2} B^L R_S }{\sqrt{\delta^2-4\tau}}\! \biggr).
  \end{align*}
\end{lemma}

\textbf{Step 4. Bias-Variance Trade-off.}
Now we are ready to finish the proof of Theorem \ref{thm:generalization}. Combining the upper bounds of $T_1$ in Lemma \ref{lemma:t1} and $T_2$ in Lemma \ref{lemma:t2} together and substituting the covering number in  Lemma \ref{lemma:covering-number}, we obtain
\begin{align*}
\EE \norm{ \hat{f} - f^*}^2_{L^2(P_{\rm data})} 
        =  O \!\biggl(\! \tau + d^{2\beta} r^{2\beta} \!+\sigma\delta+  \frac{\sigma^2 K}{n} \log\biggl( \frac{ \sqrt{dL} K^{L/2} B^L R_S }{\sqrt{\delta^2-4\tau}}  \biggr)\!\biggr),
\end{align*}
where we set approximation error to be $d^{2\beta} r^{2\beta}$. 
Plug in our choice of $(L,B,K)$, and choose $\delta = r^{2\beta}$. Then we can conclude
\begin{align*}
    \EE \norm{ \hat{f} - f^*}^2_{L^2(P_{\rm data})} =  O\biggl(   \tau +  \sigma r^{2\beta} 
    +\frac{\sigma^2}{n} \biggl(\frac{R_S}{r} \biggr)^p \log\biggl( \frac{(R_S/r)^p}{ r^{4\beta}-4\tau}\biggr) \biggr).
\end{align*}

\subsection{Proof Sketch of Lemma \ref{lemma:gaussian}}\label{sec:gaussian}

In this section, we present our basic idea to construct $S(R,r;p)$
and the proof sketch of Lemma \ref{lemma:gaussian}. For simplicity of proof, we assume $Q=I$ so that $\Sigma=\Lambda=\diag(\lambda_1^2,\ldots,\lambda_d^2)$. The detailed proof is given in Appendix \ref{proof:gaussian}, which can be easily extended to the case when $Q$ is not an identity matrix. The proof of Theorem \ref{thm:generalization-exp} is given in Appendix \ref{sec:dim-gaussian}.

Given the Gaussian sample distribution, we hope to choose some region in $S \subset \RR^d$ with high probability measure and effective Minkowski dimension $p<d$. Then we can only apply neural networks to approximate $f^*$ within each cube of the small covering of $S$ and thereby significantly reduce the network size.
In literature, it is common to truncate the ambient space within a hyper-ellipsoid for Gaussian distribution \citep{ellis2007multivariate,pakman2014exact}. Similarly, we consider the 'thick' low-dimensional hyper-ellipsoid $S(R,r;p)$ defined in \eqref{eq:set}. 
%\begin{align*}
%     S(R,r;p) := \bigg\{z \in \RR^d \bigg|& \sum_{i=1}^p \frac{z_i^2}{\lambda_i^2} \leq R^2, \notag \\
%    &|z_j| \leq \frac{r}{2} \text{ for } j=p+1,\ldots,d  \bigg\}. \label{eq:set2}
%\end{align*}
%where $R,r>0$ and $p\in \NN_+$ are independent parameters. 
Then we construct a minimal cover of $S(R,r;p)$ as a union of nonoverlapping hypercubes with side length $r$, which is equal to the thickness of $S(R,r;p)$. In particular, this cover contains multiple layers of hypercubes to cover the first $p$ dimensions of $S(R,r;p)$ while only needs one layer for the rest dimensions. Intuitively, we only learn about hypercubes that cover the first $p$ dimensions without paying extra effort to study dimensions with low probability density.

A natural question arising from the construction of $S(R,r;p)$ is how to select a proper dimension $p$. To address this problem, we first notice that each side length of the $p$-dimension hyper-ellipsoid is supposed to be greater than the side length of hypercubes $r$, i.e. $2\lambda_i R \geq r$ for $i=1,\ldots,p$, so that we would not waste hypercubes to cover dimensions with too small side length. For simplicity of calculation, we choose $p>0$ that satisfies $\lambda_p^{-1} = 2R/r$ for any given $R,r>0$.  

Now we come to prove Lemma \ref{lemma:gaussian}. Firstly, we compute the probability outside $S(R,r;p)$. By union bound, this probability can be upper bounded by two parts, the probability out of hyper-ellipsoid for the first $p$ dimensions and the probability out of hypercube with side length $r$ for the rest $d-p$ dimensions. The first part is equal to the tail bound of $p$-dimensional standard Gaussian by the construction of $S(R,r;p)$. The second part can be solved similarly by linearly transforming each dimension to be standard Gaussian.

Then we calculate the covering number of $S(R,r;p)$. Notice that the first $p$ dimensions of $S(R,r;p)$ is contained in a $p$-dimensional hyper-rectangle with side length $2\lambda_i R$ for $i=1,\ldots,p$, while only one hypercube is required to cover the $j$-th dimension for $j=p+1,\ldots,d$. Therefore, the $r$-covering number can be upper bounded by $\prod_{i=1}^p (2\lambda_i R)/r^p$.

\section{Discussion and Conclusion}

In this paper, we have presented a generic approximation and generalization theory and applied it to Gaussian Random Design. Furthermore, our theory is applicable to scenarios where data are sampled from a mixture of distributions, denoted as $P_{\rm data} = \sum_{i} w_i P_i$. Each distribution $P_i$ is assigned a weight $w_i$ and has a low-dimensional support. In such cases, we focus on a subset of distributions with significant weights, while neglecting distributions with small weights. Then the effective Minkowski dimension depends on the data support of the selected distributions. As the sample size grows, including more distributions becomes necessary to achieve higher estimation accuracy. Another example where our theory can be applied is the case of an approximate manifold, where the data are concentrated on a low-dimensional manifold. In this scenario, the effective Minkowski dimension corresponds to the intrinsic dimension of the manifold.

To illustrate the concept of effective dimension in real-world examples, we refer to a study by \citet{pope2021intrinsic}, which investigates the intrinsic dimension of several popular benchmark datasets for deep learning. The intrinsic dimension presented in \citet{pope2021intrinsic} can be seen as an approximate estimate of the Minkowski dimension, as demonstrated in \citet{levina2004maximum,grassberger1983measuring}. In our work, we adopt the methodology employed by \citet{pope2021intrinsic} and utilize generative adversarial networks trained on the ImageNet dataset to generate samples containing varying numbers of daisy images. To estimate the intrinsic dimension of these generated samples, we employ the Maximum Likelihood Estimation (MLE) method, which is achieved by computing the Euclidean distances between each data point and its $k$ nearest neighbors. The obtained results are presented in Figure \ref{figure}, which clearly demonstrates that the intrinsic dimension estimated from a finite sample of images increases as the sample size grows. This finding aligns with our theory that the effective Minkowski dimension is an increasing function of the sample size.
\begin{figure}
    \centering
    \includegraphics[width=0.61\textwidth]{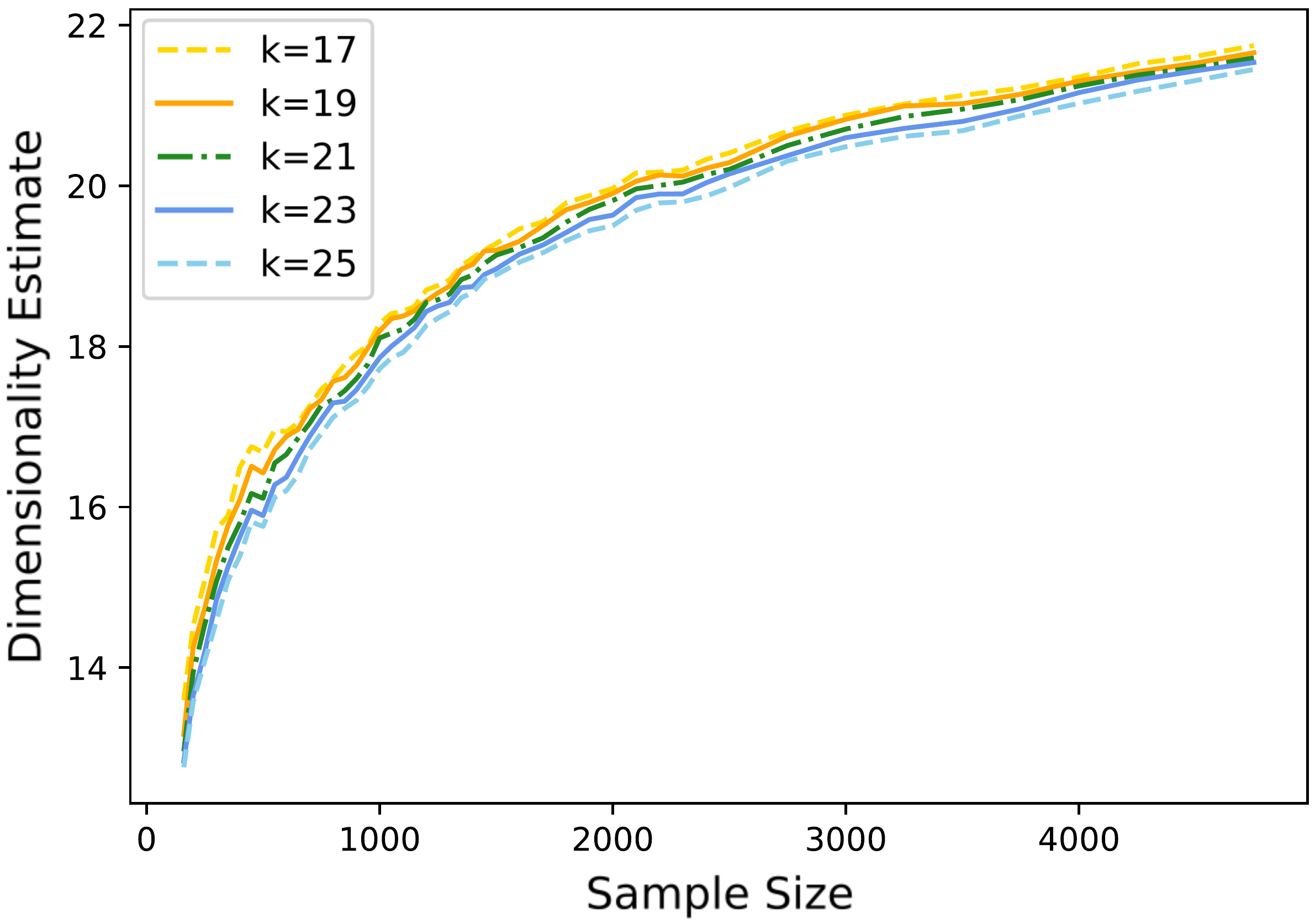}
    \caption{Dimensionality estimates of images obtained using the MLE method with $k$ nearest neighbors under different sample size.}
    \label{figure}
\end{figure}

In conclusion, this paper studies nonparametric regression of functions supported in $\RR^d$ under data distribution with effective Minkowski dimension $p<d$, using deep neural networks. Our results show that the $L^2$ error for the estimation of $f^* \in \cH(\beta,\cD_{\rm data})$ converges in the order of $n^{-2\beta/(2\beta+p)}$. To obtain an $\epsilon$-error for the estimation of $f^*$, the sample complexity scales in the order of $\epsilon^{-(2\beta+p)/\beta}$, which demonstrates that deep neural networks can capture the effective Minkowski dimension $p$ of data distribution. Such results can be viewed as theoretical justifications for the empirical success of deep learning in various real-world applications where data are approximately concentrated on a low-dimensional set.

\bibliographystyle{ims}
\bibliography{ref}

\appendix
\section{Proof of Theorem \ref{lemma:approx}}\label{appendix:approx}
In this section, we provide the omitted proof in Section \ref{sec:proof-sketch-approx}.

\subsection{Lemma \ref{lem:covering_division}}
\begin{lemma}[Lemma 20 in \citet{nakada2020adaptive}]\label{lem:covering_division}
    Let $\cC=\{I_k\}_{k=1}^{N_r(S)}$ be a minimum $r$-covering of $S$ where $I_k$'s are hypercubes with side length $r$.  Then, there exists a disjoint partition $\{\cC_j\}_{j=1}^{5^d} \subset \cC$ such that $\cC = \bigcup_{j=1}^{5^d} \cC_j$ and $d(I_i, I_l) \geq r$ hold for any $I_i \neq I_l \in \cC_j$ if $\card(\cC_j) \geq 2$, where $d(\cA,\cB):= \inf \{\norm{x-y} | x\in\cA,y\in\cB\}$ is defined as the distance of any two sets $\cA$ and 
    $\cB$.  % for all $i = 1, \dots, 5^D$.
\end{lemma}

\subsection{Realization of hypercube determination function $g_I^{\text{ind},r}$}\label{sec:realize-ind}

Hypercube determination function $g_I^{\text{ind},r}$ can be realized by weight matrices and intercepts $(4,0)\odot(W^2,-d)\odot[(W_1^1,b_1^1),\ldots,(W_d^1,b_d^1)]$ where $W_i^1,b_i^1$ and $W^2$ are defined by
\begin{align*}
    W_i^1 := \begin{pmatrix}e_i^\top & e_i^\top & e_i^\top & e_i^\top \\ 0 & 0 & 0 & 0\end{pmatrix}^\top,~~
    b_i^1 := \biggl(-\iota_i + r ~~ -\iota_i + \frac{r}{2} ~~ -\iota_i - \frac{r}{2} ~~ -\iota_i - r \biggr),
\end{align*}
and
\begin{align*}
    W^2 = \bigg(&\underbrace{\frac{2}{r}, -\frac{2}{r}, -\frac{2}{r}, \frac{2}{r},  \frac{2}{r}, -\frac{2}{r}, -\frac{2}{r}, \frac{2}{r}, \dots, \frac{2}{r}, -\frac{2}{r}, -\frac{2}{r}, \frac{2}{r}}_{4d},\frac{1}{4} \bigg).
\end{align*}
The above realization gives exactly the form in \eqref{eq:identical}. Moreover, we summarize the properties of $g_I^{\text{ind},r}$ as following:
\begin{proposition}\label{prop:ind}
    For any $x\in \cD_{\rm data}$ and $y \in \RR$, we have
    \begin{align*}
        g_I^{\text{ind},r}(x,y)
        \begin{cases}
            =y,  &x\in I \text{ and } y \in [0,4], \\
            \leq y,  &x\in I\oplus \frac{r}{2}\text{ and } y \in [0,4], \\
            =0,  &\text{otherwise}. \\
        \end{cases}
    \end{align*}
    Furthermore, we obtain the following properties
    \begin{enumerate}
        \item $L(g_I^{\text{ind},r})=3$,
        \item $B(g_I^{\text{ind},r})\leq \max\{4,d,1+r,2/r\}$,
        \item $K(g_I^{\text{ind},r})=24d+6$.
    \end{enumerate}
\end{proposition}

\subsection{Bounding the approximation error}\label{sec:approx-error-bound}

Firstly, we compute the approximation error of using $g^{f_0}_\epsilon$ to estimate $f_0=f^*+2$. Recall that we defined $g^{f_0}_\epsilon:=g^{\max,5^d} \circ g^{\text{sum}} \circ g^{\text{simul}}_{\epsilon/2}$. When $x \in S$, there exists some $I \in \cC$ such that $x \in I$. Then for this $x$, we have
\begin{align*}
    g^{f_0}_\epsilon(x) &= \max_{j \in [5^d]} \sum_{I_k \in \cC_j} g^{\text{simul}}_{\epsilon/2,k}(x)
    = \max_{I_k \in \neig(I)}  g^{\text{simul}}_{\epsilon/2,k}(x)  
    \leq \max_{I_k \in \neig(I)}  g^{\poly}_{\epsilon/2,k}(x),
\end{align*}
where $\neig(I)=\{ I' \in \cC | (I \oplus 3r/2)\cap I' \neq \emptyset \}$ denotes the $3r/2$-neighborhood of hypercube $I$. 
In other words, when computing $g^{f_0}_\epsilon(x)$, we only need to take maximum over the estimated function value within hypercubes near $x$. The second equality follows from the Proposition \ref{prop:ind} that 
 $g^{\text{simul}}_{\epsilon/2,l}(x)=0$ for $I_l \not \in \neig(I)$ and $d(I_l, I_k) > r$ holds for $I_l \neq I_k \in \cC_i$ for all $i$. The last inequality is due to the construction of $g^{\text{simul}}_{\epsilon/2}$ in \eqref{eq:simul}.

Given $\epsilon \in (0, 1)$, we ensure $0 \leq g^{\text{simul}}_{\epsilon/2,k}(x) \leq 4$ for all $I_k \in \cC$ by Proposition \ref{prop:ind}, since $g^{\text{simul}}_{\epsilon/2,k}$ approximates $f_{I_k}$ which is an $\epsilon/2$-accuracy Taylor polynomial of $f_0 \in [1, 3]$. When $x \in I_t$, the error is bounded as
\begin{align*}
    | g^{f_0}_\epsilon (x) - f_0(x) | 
    =& \max \Bigl\{  \max_{I_k \in \neig(I_t)}  g^{\text{simul}}_{\epsilon/2,k}(x) -f_0(x), f_0(x)- \max_{I_k \in \neig(I_t)}  g^{\text{simul}}_{\epsilon/2,k}(x)         \Bigr\} \\
    \leq & \max \Bigl\{  \max_{I_k \in \neig(I_t)}  g^{\poly}_{\epsilon/2,k}(x)-f_0(x), f_0(x)- g^{\text{poly}}_{\epsilon/2,t}(x) \Bigr\}\\
    \leq& \max_{I_k \in \neig(I_t)}  \Bigl| g^{\poly}_{\epsilon/2,k}(x)-f_0(x) \Bigr| \\
    \leq& \max_{I_k \in \neig(I_t)}  \Bigl| g^{\poly}_{\epsilon/2,k}(x)-f_{I_k}(x) \Bigr|
     + \max_{I_k \in \neig(I_t)}  \bigl| f_{I_k}(x) -f_0(x) \bigr| \\
    \leq& \frac{\varepsilon}{2} + d^\beta \biggl(\frac{3r}{2} \biggr)^\beta \leq \epsilon,
\end{align*}
where the last inequality follows from \eqref{eq:poly} and Lemma \ref{lemma: taylor_approx}. 
In terms of parameter tuning, we choose $r = d^{-1} \epsilon^{1/\beta}/2$. 

Next, we extend approximation results of $f_0$ to $f^*$. To do so, we firstly implement a neural network  $g^{\text{mod}}(z)=(-z+1)\circ\relu(-z+2)\circ \relu(z-1)$, which has the equivalent form $g^{\text{mod}}(z)=\min(\max(1,x),3)-2$ for any $z\in \RR$. In addition, $g^{\text{mod}}$ has the following properties:
\begin{equation*}
    L(g^{\text{mod}}) = 3,~~B(g^{\text{mod}}) \leq 2,~~\text{and  }K(g^{\text{mod}}) = 12.
\end{equation*}

Then consider $g^{f^*}_\epsilon = g^{\text{mod}}\circ g^{f_0}_\epsilon$ to obtain the desired approximation error $\epsilon$ for any $x\in S$ 
\begin{align*}
    \sup_{x \in \cD_{\rm data}} |g^{f^*}_\epsilon(x) - f^*(x)| 
    =& \sup_{x \in \cD_{\rm data}} \bigl|\min\bigl(\max\bigl(1,g^{f_0}_\epsilon(x) \bigr),3\bigr) - (f^*(x) + 2)\bigr|\\
    =& \sup_{x \in \cD_{\rm data}} \bigl|\min\bigl(\max\bigl(1,g^{f_0}_\epsilon(x) \bigr),3\bigr) - f_0(x) \bigr|\\
    \leq& \sup_{x \in \cD_{\rm data}} \bigl| g^{f_0}_\epsilon(x)  - f_0(x) \bigr|\\
    \leq& \epsilon.
\end{align*}

\subsection{Computing network sizes}\label{sec:parameter}

Recall that the ReLU neural network $g^{f^*}_\epsilon$ is defined as  
\begin{align*}
    g^{f^*}_\epsilon=& g^{\text{mod}}\circ g^{\max,5^d} \circ g^{\text{sum}} \circ g^{\text{simul}}_{\epsilon/2}\\
    =& g^{\text{mod}}\circ g^{\max,5^d} \circ g^{\text{sum}} \circ\bigl(g_{I_1}^{\text{ind},r}\circ g^{\text{filter}}_1 ,\ldots,g_{I_{ N_r(S)}}^{\text{ind},r}
    \circ g^{\text{filter}}_{ N_r(S)} \bigr) 
    \circ (g^{Id}_{d,L},  g^{\poly}_{\epsilon/2}).
\end{align*}
Note that $N_r(S) \leq c_0 r^{-p}$.
Combined with sub-neural network structures given in Appendix B.1.1 of \citet{nakada2020adaptive}, $g^{f^*}_\epsilon$ has the following properties:
\begin{align*}
    L(g^{f^*}_\epsilon(x)) &= L(g^{\text{mod}}) + L(g^{\max,5^d}) + L(g_{I_1}^{\text{ind},r}) + L(g^{\text{filter}}_1) + L(g^{\poly}_{\epsilon/2})\\
    &\leq 11 + 2d \log_2 5 + (11 + (1+\beta)/d)(2 + \log_2\beta),\\
    B(g^{f^*}_\epsilon(x)) &\leq \max\bigl\{B(g^{\text{mod}}), B(g^{\max,5^d}),B(g_{I_1}^{\text{ind},r}),B(g^{\text{filter}}_1) , B(g^{Id}_{d,L}),B(g^{\poly}_{\epsilon/2})\bigr\},\\
    &\leq \max\{d, 1 + r, 2 /r, c_1^{\poly} R_S^{\beta s_1^{\poly}}\epsilon^{-s_1^{\poly}} \},\\
    &\leq \max\{ 4d \epsilon^{-1/\beta} , c_1^{\poly} R_S^{\beta s_1^{\poly}}\epsilon^{-s_1^{\poly}} \},\\
    K(g^{f^*}_\epsilon(x)) &\leq 2K(g^{\text{mod}}) + 2K(g^{\max,5^d})  + 2N_r(S)\cdot K(g_{I_1}^{\text{ind},r}\circ g^{\text{filter}}_1) + 2 K(g^{Id}_{d,L}) + 2K(g^{\poly}_{\epsilon/2})\\
    &\leq  2 c_2^{\poly} R_S^p \epsilon^{-p/\beta} + 2(50d + 17 + c_2^{\poly})   N_r(S)\\
    &\quad+ 2(12 + 42 \times 5^d + 2d + 2d(11 + (1+\beta)/p)(2 + \log_2 \beta))),\\
    &\leq  2 c_2^{\poly} R_S^p \epsilon^{-p/\beta} + 2(50d + 17 + c_2^{\poly})c_0 (2d)^p\epsilon^{-p/\beta}  \\
    &\quad+ 2(12 + 42 \times 5^d + 2d + 2d(11 + (1+\beta)/p)(2 + \log_2 \beta))),
\end{align*}
where $c_2^{\poly}=O(d^{2+\lfloor\beta\rfloor})$. By adjusting several constants, we obtain the statement.

\section{Proof of Theorem \ref{thm:generalization}}\label{sec:proof-generalization}

The proof of Theorem \ref{thm:generalization} mainly follows \citet{chen2022nonparametric}. Our Lemma \ref{lemma:t1} and \ref{lemma:t2} is a slight revision of Lemma 5 and 6 in \citet{chen2022nonparametric}, where we substitute $\ell_\infty$ covering number with $\ell_2$ covering number to deal with unbounded domain. In this section, we compute the $\ell_2$ covering number of neural network class and then present the proof of  Theorem \ref{thm:generalization}.

\subsection{Proof of Lemma \ref{lemma:covering-number}}
\begin{proof}[Proof of Lemma \ref{lemma:covering-number}]
    To construct a covering for $\cF(H,L,B,K)$, we discretize each parameter by a unit grid with grid size $h$. Recall that we write $f \in \cF(H,L,B,K)$ as $f(x) = W_L \cdot \relu (W_{L-1}\cdots \relu(W_1 x +b_1) \cdots +b_{L-1}) +b_L$ in \eqref{eq: nn-form}. Choose any $f,f' \in \cF(H,L,B,K)$ with parameters at most $h$ apart from each other. Denote the weight matrices and intercepts in $f,f'$ as $W_L,\ldots,W_1,b_L,\ldots,L_1$ and $W_L',\ldots,W_1',b_L',\ldots,L_1'$ respectively, where $W_l \in \RR^{{d_l}\times d_{l-1}}$ and $b_l \in \RR^{d_l}$ for $l=1,\ldots,L$. Without loss of generality, we assume $d_l \leq K$ since all the parameters have at most $K$ nonzero entries. If the input dimension is larger than $K$, we let the redundant dimensions of input equal to zeros.
    
    Notice that for any random variable $y \in \RR^{d_{l-1}}$ which is subject to a distribution $P_Y$, we have
    \begin{equation*}
         \int_{\RR^{d_{l-1}}} \bignorm{ ( W_l y+ b_l)-( W_l' y+ b_l')}_2^2 \ud P_Y(y) =  \int_{\RR^{d_{l-1}}}   \bignorm{ \sum_{i=1}^{d_{l-1}} (W_{l,i} - W_{l,i}') y_i+(b_l- b_l') }_2^2 \ud P_Y(y).
    \end{equation*}
    By the inequality $\norm{t+s}^2 \leq 2\norm{t}^2 + 2\norm{s}^2$ which holds for any $s,t\in \RR^{d_l}$, we obtain
    \begin{align*}
        \int_{\RR^{d_{l-1}}} \bignorm{ ( W_l y+ b_l)-( W_l' y+ b_l')}_2^2 \ud P_Y(y) 
        \leq & 2 \int_{\RR^{d_{l-1}}}  \bignorm{\sum_{i=1}^{d_{l-1}}  (W_{l,i}-W_{l,i}') y_i  }^2 \ud P_Y(y)  + 2 \norm {b_l-b_l'}_2^2 \\
        \leq &  2 \sup_{i=1\ldots,d_{l-1}} \norm{ W_{l,i}-W_{l,i}' }_2^2 \cdot \int_{\RR^{d_{l-1}}}  \sum_{i=1}^{d_{l-1}} y_i^2  \ud P_Y(y)  + 2 \norm {b_l-b_l'}_2^2.
    \end{align*}
    Since parameters $W_l,b_l$ differ at most $h$ from $W_l',b_l'$ with respect to each entry, we get 
    \begin{align}\label{eq:difference1}
    \int_{\RR^{d_{l-1}}} \bignorm{ ( W_l y+ b_l)-( W_l' y+ b_l')}_2^2 \ud P_Y(y) 
        &\leq  2 d_{l-1} h^2 \norm{y}^2_{L_2(P_Y)} + 2 d_l h^2 \notag \\
        &\leq   2 K h^2 \norm{y}^2_{L_2(P_Y)} + 2 K h^2 .
    \end{align}
Similarly, we have
\begin{align}\label{eq:difference2}
    \int_{\RR^{d_{l-1}}} \bignorm{  W_l y+ b_l }_2^2 \ud P_Y(y) = & \int_{\RR^{d_{l-1}}}   \bignorm{ \sum_{i=1}^{d_{l-1}} W_{l,i}  y_i + b_l }_2^2 \ud P_Y(y) \notag \\
        \leq & 2 \int_{\RR^{d_{l-1}}}  \bignorm{ \sum_{i=1}^{d_{l-1}}  W_{l,i} y_i}^2 \ud P_Y(y)+ 2 \norm {b_l}_2^2 \notag  \\
        \leq &  2 \sup_{i=1\ldots,d_{l-1}} \norm{ W_{l,i} }_2^2 \cdot \int_{\RR^{d_{l-1}}}  \sum_{i=1}^{d_{l-1}} y_i^2  \ud P_Y(y)  + 2 \norm {b_l}_2^2 \notag \\
        \leq &  2 d_{l-1} B^2 \norm{y}^2_{L_2(P_Y)}  + 2 d_l B^2 \notag\\
        \leq &  2 K B^2 \norm{y}^2_{L_2(P_Y)}  + 2 K B^2.
\end{align}

Since the ReLU actiavtion function is $1$-Lipschitz continuous for each coordinate, we can apply \eqref{eq:difference1} and \eqref{eq:difference2} repeatedly to bound $ \norm{f-f'}^2_{L_2(P_{\rm data},S)}$:
\begin{align*}
  \norm{f-f'}^2_{L_2(P_{\rm data},S)}
  =& \int_{S}  \bigl|  W_L \cdot \relu (W_{L-1}\cdots \relu(W_1 x +b_1) \cdots +b_{L-1}) +b_L \\
  &\quad -  W_L' \cdot \relu (W_{L-1}'\cdots \relu(W_1' x +b_1') \cdots +b_{L-1}') -b_L'    \bigr|^2 \ud P_{\rm data}(x) \\
  \leq& 2 \int_{S}   \bigl|  W_L \cdot \relu (W_{L-1}\cdots \relu(W_1 x +b_1) \cdots +b_{L-1}) +b_L \\
  &\quad -  W_L' \cdot \relu (W_{L-1}\cdots \relu(W_1 x +b_1) \cdots +b_{L-1}) -b_L'  \bigr|^2 \ud P_{\rm data}(x)\\
  &\quad + 2 \norm{W_L'}_2^2 \int_{S}   \bigl\|   (W_{L-1}\cdots \relu(W_1 x +b_1) \cdots +b_{L-1})  \\
  &\quad -   (W_{L-1}' \cdots \relu(W_1' x +b_1') \cdots +b_{L-1}')  \bigr\|_2^2 \ud P_{\rm data}(x)\\
  \leq & 4 K h^2+ 4 K h^2  \int_{S}   \norm{  W_{L-1}\cdots \relu(W_1 x +b_1) \cdots +b_{L-1} }^2 \ud P_{\rm data}(x)\\ 
  &\quad +  2 K B^2 \int_{S}   \bigl\|   (W_{L-1}\cdots \relu(W_1 x +b_1) \cdots +b_{L-1})  \\
  &\quad -   (W_{L-1}' \cdots \relu(W_1' x +b_1') \cdots +b_{L-1}')  \bigr\|_2^2 \ud P_{\rm data}(x).
\end{align*}
Besides, we derive the following bound on $\norm{ W_{L-1}\cdots \relu(W_1 x +b_1) \cdots +b_{L-1}}_{L^2(P_{\rm data},S)}$:
\begin{align*}
    \int_{S}   &\norm{  W_{L-1}\cdots \relu(W_1 x +b_1) \cdots +b_{L-1} }^2 \ud P_{\rm data}(x) \\
    &\quad \leq  2KB^2 \int_{S}   \norm{  W_{L-2}\cdots \relu(W_1 x +b_1) \cdots +b_{L-2} }^2 \ud P_{\rm data}(x) +2KB^2 \\
    &\quad \leq (2KB^2)^{L-1} d R_S^2 +  (2KB^2)^{L-1} \\
    &\quad \leq 2^L(KB^2)^{L-1} d R_S^2,
\end{align*}
where the last inequality is derived by induction and $\|x \|^2 =\sum_{i=1}^d  x_i^2 \leq d R_S^2$ for any $x\in S$. Substituting back into the bound for $ \norm{f-f'}^2_{L_2(P_{\rm data},S)}$ , we obtain
\begin{align*}
    \norm{f-f'}^2_{L_2(P_{\rm data},S)} \leq& 4 K h^2+ 2^{L+2} K^L B^{2(L-1)} h^2 d R_S^2 \\
    & \quad+  4 K B^2 \int_{S}   \bigl\|   (W_{L-1}\cdots \relu(W_1 x +b_1) \cdots +b_{L-1})  \\
    &\quad -   (W_{L-1}' \cdots \relu(W_1' x +b_1') \cdots +b_{L-1}')  \bigr\|_2^2 \ud P_{\rm data}(x)\\
    \leq & 4 (L-1) (K B^2)^{L-1} h^2+ 2^{L+2} (L-1) K^L B^{2(L-1)} h^2 d R_S^2 \\
    & \quad + ( 2 K B^2)^{L-1} \int_{S}   \norm{ W_1 x +b_1  -   W_1' x -b_1' }_2^2 \ud P_{\rm data}(x)\\
    \leq & 4^{L-1} L K^L B^{2L} h^2 d R_S^2,
\end{align*}
where the second inequality is obtained by induction. Therefore, combining the above inequality with $\|f\|_\infty \leq 1$ for any $f \in\cF(L,B,K)$ and $P_{\rm data}(S^c) \leq \tau$, we get
\begin{align*}
    \norm{f-f'}^2_{L_2(P_{\rm data})} =& \int_{S}  | f(x) - f'(x) |^2  \ud P_{\rm data}(x) +  \int_{S^c}  | f(x) - f'(x) |^2  \ud P_{\rm data}(x) \\
    \leq &  4^{L-1} L K^L B^{2L-2} h^2 d R_S^2 + 4\tau.
\end{align*}
Now we choose $h$ satisfying $h=\sqrt{(\delta^2-4\tau)/(4^{L-1}L K^L B^{2L-2} d R_S^2)}$. Then discretizing each parameter uniformly into $2B/h$ grid points yields a $\delta$-covering on $\cF(L, B,K)$.  Moreover, the covering number $\cN_2(\delta, \cF(L, B,K))$ satisfies
\begin{align*}
    \log \cN_2(\delta, \cF( L, B,K)) \leq K \log\biggl( \frac{2B}{h} \biggr) = K \log\biggl( \frac{2^L \sqrt{dL} K^{L/2} B^L R_S }{\sqrt{\delta^2-4\tau}} \biggr).
\end{align*}
\end{proof}

\subsection{Proof of Theorem \ref{thm:generalization}}
\begin{proof}[Proof of Theorem \ref{thm:generalization}]
The square error of the estimator $\hat{f}$ can be decomposed into a squared bias term and a variance term, which can be bounded using the covering number of the function class. According to Lemmas \ref{lemma:t1} and \ref{lemma:t2}, for any constant $\delta\in(0,1)$, we have
\begin{equation}\label{eq: l2-bound}
    \begin{split}
        \EE \| \hat{f} - f^*\|^2_{L^2(P_{\rm data})} \leq& c \inf_{f \in \cF( L, B, K)} \norm{ f(x) - f^*(x)}^2_{L^2(P_{\rm data})}  + c \sigma^2 \frac{\log \cN_2(\delta, \cF( L, B, K)) + 2}{n} \\
        &\quad +  c\bigg(   \sqrt{\frac{\log \cN_2(\delta, \cF(L, B, K)) + 2}{n}} + 1 \bigg)\sigma \delta + \frac{c'}{3n} \log \cN_2(\delta/4H, \cF(L, B, K))+ c'\delta.
    \end{split}
\end{equation}

%Next, by Assumption \ref{assumption:pdata_cover}, for any sufficiently small $r,\tau>0$, there exists $p=p(r,\tau)$  and $S\subseteq [0,1]^d$ satisfying $P_{\rm data}(S) \geq 1-\tau$ and $N_r(S) \leq c_0 r^{-p}$. Besides, notice that for any $f \in \mathcal{F}(L,B,K)$, the square error of $f$ can be divided into one part of the high-density region $S_\epsilon$ and the other of the rest area $S^c$. Thereby we get 
%\begin{align*}
 %   \int_{\RR^d} (f(x)-f^*(x))^2 \ud P_{\rm data}(x) =&  \int_{S_\epsilon^c} (f(x)-f^*(x))^2 \ud P_{\rm data}(x) +\int_{S_\epsilon} (f(x)-f^*(x))^2 \ud P_{\rm data}(x) \\
 %   \leq& 4\tau + \int_{S_\epsilon} (f(x)-f^*(x))^2 \ud P_{\rm data}(x) ,
%\end{align*}
%where the last inequality is given by $P_{\rm data}(S^c) \leq \tau$. 
Choose $\epsilon=(2dr)^\beta $ in Theorem \ref{lemma:approx}.
Accordingly, we set tuple $(L,B,K)$ as
\begin{equation*}
     L=C_1,~~ B=O(R_S^{\beta s} r^{-s\beta}),~~ \text{and}~ K=C_2 R_S^p r^{-p}.
\end{equation*}
Then we have
%there exists a network class $\mathcal{F}(L,B,K)$ that yields a function $f$ approximating $f^*$ up to an error $\epsilon$ on $S$, i.e. 
\begin{equation*} 
    \inf_{f \in \cF( L, B, K)} \norm{ f(x) - f^*(x)}^2_{L^2(P_{\rm data})} \leq (2dr)^{2\beta} + 4\tau. 
\end{equation*} 
Invoking the upper bound of the covering number in Lemma \ref{lemma:covering-number}, we derive
\begin{align*}
    \EE \norm{ \hat{f} - f^*}^2_{L^2(P_{\rm data})} \leq& (2dr)^{2\beta} + 4\tau +  \frac{c\sigma^2}{n}  \bigg(  K \log\biggl( \frac{2^{L} \sqrt{dL} K^{L/2} B^L R_S }{\sqrt{\delta^2-4\tau}} \biggr) +2c \bigg) \\
        &\quad + c\sqrt{\frac{ K \log(2^{L} \sqrt{dL} K^{L/2} B^L R_S /\sqrt{\delta^2-4\tau}) +1}{n}} \sigma\delta \\
        &\quad + \frac{c'}{3n}   K \log\biggl( \frac{2^{L} \sqrt{dL} K^{L/2} B^L R_S }{\sqrt{\delta^2-4\tau}}  \biggr) +(c\sigma +c')\delta \\
        =& O\bigg( \tau + d^{2\beta} r^{2\beta} +  \frac{\sigma^2 K}{n} \log\biggl( \frac{ \sqrt{dL} K^{L/2} B^L R_S }{\sqrt{\delta^2-4\tau}}  \biggr) \\
        &\quad + \biggl(  \frac{K \log (\sqrt{dL} K^{L/2} B^L R_S /\sqrt{\delta^2-4\tau} ) }{n} \biggr)^{1/2}  \sigma \delta+ \sigma \delta+ \frac{1}{n}\bigg).
\end{align*}
By Cauchy-Schwartz inequality, for $0<\delta<1$, we have
\begin{align*}
    \EE \norm{ \hat{f} - f^*}^2_{L^2(P_{\rm data})} 
        =& O\biggl( \tau + d^{2\beta} r^{2\beta} +  \frac{\sigma^2 K}{n} \log\biggl( \frac{ \sqrt{dL} K^{L/2} B^L R_S }{\sqrt{\delta^2-4\tau}}  \biggr)+\sigma\delta+ \frac{1}{n}\biggr).
\end{align*}
Plugging in our choice of $(L,B,K)$, we get
\begin{align*}
    \EE \norm{ \hat{f} - f^*}^2_{L^2(P_{\rm data})} 
    = & O\biggl( \tau + d^{2\beta} r^{2\beta} + \frac{\sigma^2 d}{n}\biggl(\frac{R_S}{r} \biggr)^p \log\biggl( \frac{(R_S/r)^p}{ \delta^2-4\tau}\biggr)+ \sigma \delta + \frac{1}{n}\biggr)
\end{align*}
Now we choose  $\delta=r^{2\beta}$. Then we deduce the desired estimation error bound
\begin{align*}
    \EE \| \hat{f} - f^*\|^2_{L^2(P_{\rm data})} &= O\biggl(   \tau +  \sigma r^{2\beta} + \frac{\sigma^2}{n}  \biggl(\frac{R_S}{r} \biggr)^p \log\biggl( \frac{(R_S/r)^p}{ r^{4\beta}-4\tau}\biggr) + \frac{1}{n}\biggr)\\
    & = O\biggl(   \tau +  \sigma r^{2\beta} + \frac{\sigma^2}{n}  \biggl(\frac{R_S}{r} \biggr)^p \log\biggl( \frac{(R_S/r)^p}{ r^{4\beta}-4\tau}\biggr)\biggr).
\end{align*}
The last equality is due to $R_S/r>1$.
\end{proof}

\section{Proof of Lemma \ref{lemma:gaussian}}\label{proof:gaussian}

\begin{proof}[Proof of Lemma \ref{lemma:gaussian}]
    For simplicity of proof, set $Q=I$. By the construction of $S(R,r;p)$ in \eqref{eq:set}, we notice that
    \begin{align*}
        \PP(X \notin S(R,r;p) ) \leq & \PP\bigg(\sum_{i=1}^p \frac{x_i^2}{\lambda_i^2} >  R^2 \bigg) + \PP\bigg(|x_j|>\frac{r}{2} \text{ for some } j \in [p+1,d]\bigg).
    \end{align*}
    Since $X_{1:p}=(x_1,\ldots,x_p)\sim N(0,\Lambda_p)$ where $\Lambda_p=\diag(\lambda_1^2,\ldots,\lambda_p^2)$, the variable $Z=\Lambda_p^{-1/2} X_{1:p} \sim N(0,I_p)$ is a standard Gaussian. Then for any fixed $R>0$, the probability $\PP(\sum_{i=1}^p x_i^2/\lambda_i^2 >  R^2 )
    $ is equal to $ \PP ( \|Z\|^2 > R^2 )$. Moreover, by Lemma \ref{lemma:gaussian-tail}, if we choose $R^2>p$, we will have
    \begin{align*}
         \PP\big( \|Z\|^2 > R^2 \big) \leq&  \bigg( \frac{2R^2 +p}{p} \bigg)^{\frac{p}{2}} \exp \bigg(-\frac{R^4}{2R^2+p} \bigg) \\
         =&  \exp \bigg(-\frac{p}{2} \cdot \frac{R^4/p^2}{R^2/p+1/2} + \frac{p}{2} \cdot \log\bigg( \frac{2R^2}{p}+1 \bigg) \bigg) \\
         =& O \Bigg( \bigg[ \exp \bigg(-\frac{2R^2}{3p} + \log\bigg( \frac{3R^2}{p} \bigg) \bigg) \bigg]^{\frac{p}{2}} \Bigg).
    \end{align*} 
    Besides, by Lemma \ref{lemma:gaussian-tail-1}, for $j=p+1,\ldots,d$, we derive 
    \begin{align*}
        \PP\bigg(|x_j|>\frac{r}{2} \bigg) = \PP\bigg( \bigg| \frac{x_j}{\lambda_j} \bigg|>\frac{r}{2\lambda_i} \bigg) 
        =  O\bigg( \exp\bigg\{ -\frac{r^2}{8\lambda_j^2} \bigg\} \bigg).
    \end{align*}
    Then we can apply the union bound of probability to get
    \begin{align*}
        \PP\bigg(|x_j|> \frac{r}{2} \text{ for some } j \in [p+1,d]\bigg) \leq \sum_{j=p+1}^d  \PP\bigg(|x_j| > \frac{r}{2} \bigg) 
        =  O\bigg( \sum_{j=p+1}^d \exp\bigg( -\frac{r^2}{8\lambda_j^2} \bigg) \bigg).
    \end{align*}
    Recall that we choose $\lambda_p^{-1}=2R/r$. Then we have
    \begin{align*}
        \PP\bigg(|x_j|> \frac{r}{2} \text{ for some } j \in [p+1,d]\bigg) 
        = & O\bigg( \sum_{j=p+1}^d \exp\bigg( -\frac{\lambda_p^2}{2\lambda_j^2} R^2 \bigg) \bigg) \\
        \leq & O\bigg( \sum_{j=p+1}^d \exp\bigg( -\frac{R^2}{2} \bigg) \bigg)\\
        = & O\bigg( \exp\bigg( -\frac{R^2}{2}+\log(d-p) \bigg) \bigg),
    \end{align*}
    where the inequality comes from $\lambda_j^2 \leq \lambda_p^2$ for $j=p+1,\ldots,d$. Therefore, we have
    %Now we plug in the exponential decay rate of eigenvalue $\lambda_i^2$'s in Assumption \ref{assump:eigen-decay} and thus arrive at
    %\begin{align*}
    %    \PP\bigg(|x_i|> \frac{r}{2} \text{ for some } i \in [p+1,d]\bigg) 
    %    = & O\bigg( \sum_{i=p+1}^d \exp\bigg( -\frac{R^2}{2} \exp(2b(i-p)) \bigg) \bigg) \\
    %\end{align*}
    \begin{equation*}
        \PP(X \notin S(R,r;p) ) = O \Bigg( \bigg[ \exp \bigg(-\frac{2R^2}{3p} + \log\bigg( \frac{3R^2}{p} \bigg) \bigg) \bigg]^{\frac{p}{2}} \Bigg) + O\bigg( \exp\bigg( -\frac{R^2}{2}+\log(d-p) \bigg) \bigg).
    \end{equation*}  
    
    %Let $R^2=r^2\exp\{b'd\}$. Combine \eqref{eq:prob1} and \eqref{eq:prob2} and we will obtain
    %\begin{align*}
         %\PP(X \notin E_p(R,r) ) &=  O\bigg( \bigg(\frac{R^2}{2} \bigg)^{\lfloor \frac{p}{2} \rfloor}\exp\bigg\{-\frac{R^2}{2} \bigg\} \bigg) + O\bigg(  \exp\bigg\{ -\frac{r^2}{a'} \exp\{b' d\} \bigg\} \bigg)\\
         %&= O\bigg(  \exp\bigg\{ -c r^2 \exp\{b' d\} \bigg\} \bigg).
    %\end{align*}
    Next, we compute the covering number of $S(R,r;p)$ using hypercubes with side length $r>0$, which is denoted as $N_r(S(R,r;p))$. Notice that the first-$p$-dimensional hyper-ellipsoid of $S(R,r;p)$ is contained in a $p$-dimensional hyper-rectangle with side length $2\lambda_i R$ for $i=1,\ldots,p$, while only one hypercube is required to cover the $j$-th dimension for $j=p+1,\ldots,d$. With this observation,  we derive the upper bound for $N_r(S(R,r;p))$:
    \begin{align*}
        N_r(S(R,r;p)) \leq \prod_{i=1}^p \bigg( \frac{2\lambda_i R}{r}\bigg) 
        = \prod_{i=1}^p \bigg( \frac{\lambda_i}{\lambda_p}\bigg),
    \end{align*} 
    where the last equality results from our choice of $p$.
    %Now we substitute $R=r\exp\{b'd/2\}$ and utilize the decay rate of eigenvalues. Then we have
    %\begin{align*}
        %N_r(E_p(R,r)) \leq \frac{\pi^{p/2}}{\Gamma(p/2+1)}\exp\{cpd\}   \cdot a^{p/2} \exp\{-bp(p+1)/4\} 
        %= O(\exp\{ cpd\}) .
    %\end{align*}
    %If we further set $r= \exp\{-cd\}$, we will get
    %\begin{equation*}
    %    \PP(X \notin E_p(R,r) ) = O(\exp\{-r^\alpha\}) \text{  and  } N_r(E_p(R,r))= O(r^{-p}).
    %\end{equation*}
    %where $R=\exp\{-(c-b'/2)d\}$ and $c>b'-2$.
    %Therefore, we obtain $ N_r(E_p(R)) = \tilde{O} (r^{-p_0}) $.
\end{proof}

\section{Proof of Theorem \ref{thm:generalization-exp}}\label{sec:dim-gaussian}

\subsection{Generalization error under exponential eigenvalue decay}\label{sec:dim-exp}
    
    Combining the criteria $\lambda_p^{-1}=2R/r$ and the exponential eigenvalue decay in Assumption \ref{assump:eigen-decay-exp}, we have
    \begin{equation*}
        \frac{1}{\mu} \exp(\theta p ) = \frac{2R}{r}.
    \end{equation*}
    Moreover, by Lemma \ref{lemma:gaussian}, we can compute the covering number of $S(R,r,p)$:
    \begin{align*}
        N_r(S(R,r;p)) &\leq \prod_{i=1}^p \bigg( \frac{\lambda_i}{\lambda_p}\bigg) =  \prod_{i=1}^p \exp(\theta(p-i)) \leq \exp\biggl( \frac{\theta p^2}{2} \biggr) =  \biggl( \frac{2\mu R}{r} \biggr)^{p/2},
    \end{align*}
    which indicates that effective Minkowski dimension of $P_{\rm data}$ is at most $p/2$. Let $r=n^{-(1-\eta)/(2\beta+\sqrt{\log n /\theta})}$ and $R=\log n$ where $\eta\in(0,1)$ is an arbitrarily small constant. Then we obtain
    \begin{align*}
        \theta p = \log\biggl( \frac{2\mu R}{r} \biggr)=  \frac{1}{\theta} \log( 2\mu)+ \log( \log n )+ \frac{(1-\eta)\log n}{2\beta +\sqrt{\log n /\theta}}  \leq \frac{2(1-\eta)\log n}{2\beta +\sqrt{\log n /\theta}}.
    \end{align*}
    Thereby, we can compute the probability ourside $S(R,r;p)$:
    \begin{align*}
        \PP(X \notin S(R,r;p) ) =& O \Biggl( \biggl[ \exp \bigg(-\frac{2R^2}{3p} + \log\bigg( \frac{3R^2}{p} \bigg) \bigg) \biggr]^{\frac{p}{2}}  + \exp\biggl( -\frac{R^2}{2}+\log(d-p)  \biggr) \Biggr) = O(n^{-\log n/3}).
    \end{align*}
    
    Apply Theorem \ref{thm:generalization} with our choice of $R$, $r$ and $p$. Accordingly, the tuple $(L,B,K)$ is set as
    \begin{equation*}
        L= C_1,~~B=O\biggl(n^{\frac{\beta s}{2\beta+\sqrt{\log n/\theta}}} (\log n)^{\beta s} \biggr), ~~ \text{and}~K=O\biggl(n^{\frac{\sqrt{\log n/\theta}}{2\beta + \sqrt{\log n/\theta} } } \biggr).
    \end{equation*}
    Then we can get
    \begin{align*}
        \EE \norm{ \hat{f} - f^*}^2_{L^2(P_{\rm data})} &= O\biggl(  \PP(X \notin S(R,r;p) ) + \sigma r^{2\beta} + \frac{\sigma^2}{n} \cdot \biggl(\frac{R_S}{r} \biggr)^{p/2} \log\biggl( \frac{(R_S/r)^{p/2}}{ r^{4\beta}-4\tau}\biggr) \biggr)\\
        &= O\biggl(  \sigma n^{- \frac{2\beta(1-\eta)}{2\beta+\sqrt{\log n/\theta}}} + \frac{\sigma ^2}{n} \cdot n^ \frac{(1-\eta)^2\log n /\theta}{ (2\beta +\sqrt{\log n /\theta})^2} \cdot (\log n)^\frac{(1-\eta)\log n /\theta}{2\beta +\sqrt{\log n /\theta}} \cdot \frac{\log n/\theta}{2\beta +\sqrt{\log n /\theta}}\cdot \log n \biggr)\\
        &= O\biggl( \sigma  n^{- \frac{2\beta(1-\eta)}{2\beta+\sqrt{\log n/\theta}}} + \sigma^2  n^{-1 + \frac{(1-\eta)^2\log n /\theta}{ (2\beta +\sqrt{\log n /\theta})^2} + \frac{(1-\eta)\log(\log n)/\theta}{2\beta +\sqrt{\log n /\theta}} }   \cdot (\log n)^{3/2} \biggr).
    \end{align*}
    The last equality utilizes the fact that $(\log n)^{\log n}=n^{\log(\log n)}$. Furthermore, notice that for sufficiently large $n$ satisfying $\log(\log n) / \sqrt{\theta \log n} \leq \eta$, we have
    \begin{align*}
        \frac{(1-\eta)^2\log n /\theta}{ (2\beta +\sqrt{\log n /\theta})^2} + \frac{(1-\eta)\log(\log n)/\theta}{2\beta +\sqrt{\log n /\theta}} &\leq  \frac{(1-\eta)^2\log n /\theta}{ (2\beta +\sqrt{\log n /\theta})^2} + \frac{(1-\eta)\eta \sqrt{\log n/\theta}}{2\beta +\sqrt{\log n /\theta}} \\
        &\leq    \frac{(1-\eta)^2 \sqrt{\log n /\theta} }{ 2\beta +\sqrt{\log n /\theta}} + \frac{(1-\eta)\eta \sqrt{\log n/\theta}}{2\beta +\sqrt{\log n /\theta}}\\
        &\leq \frac{(1-\eta) \sqrt{\log n/\theta}}{2\beta +\sqrt{\log n /\theta}}.
    \end{align*}
    %\begin{align*}
    %    -1 + \frac{\log n /\theta}{ (2\beta +\sqrt{\log n /\theta})^2} &= - \frac{4\beta^2 + 4\beta\sqrt{\log n /\theta} }{ (2\beta +\sqrt{\log n /\theta})^2} \\
    %    & = -\frac{2\beta}{ 2\beta +\sqrt{\log n /\theta}} \cdot \frac{2\beta + 2 \sqrt{\log n /\theta} }{ 2\beta +\sqrt{\log n /\theta}} \\
    %    &\leq -\frac{2\beta}{ 2\beta +\sqrt{\log n /\theta}}. 
    %\end{align*}
    %Also, we observe that 
    %\begin{align*}
    %     (\log n)^\frac{\log n /\theta}{2\beta +\sqrt{\log n /\theta}} &= \exp\biggl( \frac{\log(\log n)\cdot\log n /\theta}{2\beta +\sqrt{\log n /\theta}} \biggr)\\
    %     &\leq \exp\Bigl( \log(\log n)\cdot\sqrt{\log n /\theta} \Bigr)\\
    %     &\leq  \exp(\eta \log n) = n^\eta
    %\end{align*}
     Therefore, we use the above observation to derive the following upper bound for generalization error:
    \begin{align*}
         \EE \norm{ \hat{f} - f^*}^2_{L^2(P_{\rm data})} = O\Bigl( \sigma^2 n^{- \frac{2\beta(1-\eta)}{2\beta+\sqrt{\log n/\theta}}}(\log n)^{3/2} \Bigr).
    \end{align*}

\subsection{Generalization error under polynomial eigenvalue decay}\label{sec:dim-poly}
    
     Similarly to last section, we firstly combine the criteria $\lambda_p^{-1}=2R/r$ and the polynomial eigenvalue decay in Assumption \ref{assump:eigen-decay-poly},     \begin{equation*}
         p = \biggl( \frac{2\rho R}{r} \biggr)^{1/\omega}.
    \end{equation*}
    Moreover, by Lemma \ref{lemma:gaussian}, we can compute the covering number of $S(R,r,p)$:
    \begin{align*}
        N_r(S(R,r;p)) &\leq \prod_{i=1}^p \bigg( \frac{\lambda_i}{\lambda_p}\bigg) =  \prod_{i=1}^p \biggl(\frac{p}{i}\biggr)^\omega = \biggl(\frac{p^p}{p!}\biggr)^\omega \leq \biggl(\frac{p^p}{p^{p/2}}\biggr)^\omega =  p^{\omega p/2} =  \biggl(\frac{2\rho R}{r}\biggr)^{p/2} ,
    \end{align*}
    which indicates that effective Minkowski dimension of $P_{\rm data}$ is at most $p/2$. Let $r=n^{-1/(2\beta+ n^\kappa)}$ and $R=n^{1/(2\omega\beta+ \omega n^\kappa)}$ with $\kappa = (1+1/\omega)/\omega$. Then we obtain
    \begin{align*}
        p = \biggl( \frac{2\rho R}{r} \biggr)^{1/\omega} =  n^{\frac{(1+1/\omega)/\omega}{2\beta+ n^\kappa}} = n^{\frac{\kappa}{2\beta+ n^\kappa}}.
    \end{align*}
    Thereby, we can compute the probability outside $S(R,r;p)$:
    \begin{align*}
        \PP(X \notin S(R,r;p) ) =& O \Biggl( \biggl[ \exp \bigg(-\frac{2R^2}{3p} + \log\bigg( \frac{3R^2}{p} \bigg) \bigg) \biggr]^{\frac{p}{2}}  + \exp\biggl( -\frac{R^2}{2}+\log(d-p)  \biggr) \Biggr) = O\Bigl(\exp\Bigl(-n^{\frac{2}{(2\omega\beta+ \omega n^\kappa)}}/3\Bigr) \Bigr).
    \end{align*}
    
    Apply Theorem \ref{thm:generalization} with our choice of $R$, $r$ and $p$. Accordingly, the tuple $(L,B,K)$ is set as
    \begin{equation*}
        L= C_1,~~B=O\Bigl(n^{\frac{(1+1/\omega)\beta s}{2\beta+ n^\kappa } } \Bigr), ~~ \text{and}~K=O\Bigl(n^{\frac{(1+1/\omega) n^{\kappa/(2\beta+n^\kappa)}}{4\beta + 2n^\kappa}} \Bigr).
    \end{equation*}
    Then we have
    \begin{equation}\label{eq:poly-bound}
    \begin{split}
        \EE \norm{ \hat{f} - f^*}^2_{L^2(P_{\rm data})} &= O\biggl(  \PP(X \notin S(R,r;p) ) + \sigma  r^{2\beta} + \frac{\sigma^2}{n} \cdot \biggl(\frac{R_S}{r} \biggr)^{p/2} \log\biggl( \frac{(R_S/r)^{p/2}}{ r^{4\beta}-4\tau}\biggr) \biggr)\\
        &= O\biggl(\sigma n^{- \frac{2\beta}{2\beta+ n^\kappa}} + \frac{\sigma ^2}{n} \cdot n^{ \frac{1+1/\omega}{2\beta+ n^\kappa}\cdot\frac{1}{2}{n^{\frac{\kappa}{2\beta+ n^\kappa}}}} \cdot n^{\frac{\kappa}{2\beta+ n^\kappa}} \log n\biggr)\\
        &= O\Bigl( \sigma  n^{- \frac{2\beta}{2\beta+ n^\kappa}} + \sigma^2 n^{-1+ \frac{1+1/\omega}{2\beta+ n^\kappa}\cdot\frac{1}{2}{n^{\frac{\kappa}{2\beta+ n^\kappa}}} + \frac{\kappa}{2\beta+ n^\kappa}} \log n\Bigr).
    \end{split}
    \end{equation}
    Notice that
    \begin{align*}
        \frac{1+1/\omega}{2\beta+ n^\kappa}\cdot\frac{1}{2}{n^{\frac{\kappa}{2\beta+ n^\kappa}}} + \frac{\kappa}{2\beta+ n^\kappa} &= \frac{1+1/\omega}{2\beta+ n^\kappa} \biggl( \frac{1}{2}{n^{\frac{\kappa}{2\beta+ n^\kappa}}} + \frac{1}{\omega}   \biggr) \\
        & \leq \frac{2}{2\beta+ n^\kappa} \biggl( \frac{1}{2}{n^{\frac{\kappa}{2\beta+ n^\kappa}}} +1 \biggr) \\
        &\leq \frac{n^\kappa}{2\beta+ n^\kappa} , 
    \end{align*}
    where the first inequality is due to $\omega>1$. Therefore, plug the above inequality in \eqref{eq:poly-bound}, we derive the following upper bound for generalization error:
    \begin{align*}
         \EE \norm{ \hat{f} - f^*}^2_{L^2(P_{\rm data})} = O\Bigl( \sigma^2 n^{- \frac{2\beta}{2\beta+ n^\kappa}} \log n\Bigr).
    \end{align*}

\section{Auxiliary Lemmas}

In this section, we investigate the probability tail bound of standard Gaussian variable, which is useful for the proof of Lemma \ref{lemma:gaussian}.  At first, we compute the tail bound for multivariate Gaussian variable.
\begin{lemma}\label{lemma:gaussian-tail-1}
    Suppose $Z=(z_1,\ldots,z_p)\sim N(0,I_p)$ is a standard Gaussian variable in $\RR^p$. Then for any $t>0$, we have
    \begin{equation*}
       \PP\big( \|Z\| > t \big) \leq \bigg( \frac{2t^2 +p}{p} \bigg)^{\frac{p}{2}} \exp \bigg(-\frac{t^4}{2t^2+p} \bigg).
    \end{equation*}
\end{lemma}
\begin{proof}
    By the Markov's inequality, for any $\mu\in(0,1/2)$, we have
    \begin{align*}
        \PP\big( \|Z\| > t \big) =& \PP \big(\exp(\mu \|Z\|^2) > \exp(\mu t^2) \big) \\
        \leq & \frac{\EE \exp(\mu \|Z\|^2) }{\exp(\mu t^2)} \\
        =&  \frac{ \prod_{i=1}^p \EE \exp(\mu z_i^2) }{\exp(\mu t^2)} ,
    \end{align*}
    where the last equality comes from the independence of $z_i$'s.
To bound $\EE \exp(\mu z_i^2)$, we first examine the moment generating function of $z_i$: for any $t\in\RR$,
\begin{align*}
    \EE \exp(t z_i) = \int_\RR \exp(tw)\phi(w) \ud w = \exp(t^2/2),
\end{align*}
where $\phi(w)=(2\pi)^{-p/2} \exp(-w^2/2)$ denotes the probability density function of stardard Gaussian. Then multiply $\exp(-t^2/(2\mu))$ on both sides,
\begin{align*}
    \int_\RR \exp\bigg(tw-\frac{t^2}{2\mu} \bigg)\phi(w) \ud w = \exp\bigg(\frac{t^2(\mu-1)}{2\mu} \bigg).
\end{align*}
By integrating both sides with respect to $t$, we have
\begin{equation*}
    \sqrt{2\pi\mu}  \int_\RR \exp\bigg( \frac{\mu w^2}{2} \bigg)\phi(w) \ud w = \sqrt{\frac{2\pi\mu}{1-\mu} },
\end{equation*}
which indicates
\begin{align*}
    \EE \exp(\mu z_i^2) = \EE \exp\bigg( \frac{2\mu z_i^2}{2} \bigg) = \sqrt{\frac{1}{1-2\mu}}.
\end{align*}
Therefore, for any $\mu \in (0,1/2)$, we have
\begin{align*}
    \PP\big( \|Z\| > t \big) \leq (1-2\mu)^{-\frac{p}{2}}\exp(-\mu t^2).
\end{align*}
Let $\mu = t^2/(2t^2+p)$ and thereby we can conclude the proof of the lemma.

\end{proof}

For standard Gaussian in $\RR$, we derive a tighter upper bound in the following lemma.
\begin{lemma}\label{lemma:gaussian-tail}
    Suppose $z\sim N(0,1)$ is a standard Gaussian variable in $\RR$. Then for any $t>0$, we have
    \begin{equation*}
       \PP\big( \norm{z} > t \big) \leq  \exp\bigg(-\frac{1}{2}t^2\bigg).
    \end{equation*}
\end{lemma}
\begin{proof}
    Firstly, for any $t>0$, compute the probability that $z>t$:
    \begin{align*}
        \PP( z > t ) =& \int_t^\infty \frac{1}{\sqrt{2\pi}}\exp\bigg( -\frac{1}{2}z^2 \bigg) \ud z \\
        = & \int_0^\infty \frac{1}{\sqrt{2\pi}}\exp\bigg( -\frac{1}{2}(u+t)^2 \bigg) \ud u \\
        =& \exp\bigg(-\frac{1}{2}t^2 \bigg)\int_0^\infty \exp(-tu)\cdot\frac{1}{\sqrt{2\pi}}\exp\bigg( -\frac{1}{2}u^2 \bigg) \ud u \\
        \leq & \exp\bigg(-\frac{1}{2}t^2 \bigg)\int_0^\infty \frac{1}{\sqrt{2\pi}}\exp\bigg( -\frac{1}{2}u^2 \bigg) \ud u \\
        = & \frac{1}{2}\exp\bigg(-\frac{1}{2}t^2 \bigg).
    \end{align*}
    Then notice that
\begin{align*}
    \PP\big( \norm{z} > t \big) = \PP( z > t ) + \PP( z < -t ) = 2\PP( z > t ).
\end{align*}
Thereby, we can conclude the proof.
\end{proof}

\end{document}